\documentclass[11pt]{article}
\usepackage{amsthm}
\usepackage{graphicx} 
\usepackage{array} 
\usepackage{mathtools}
\usepackage{amsmath, amssymb, amsfonts, verbatim,bbm}
\usepackage{hyphenat,epsfig,subcaption,multirow}
\usepackage{nicefrac}
\usepackage{paralist}
\usepackage[utf8]{inputenc}
\usepackage[font=small, labelfont=bf]{caption}
\usepackage{wrapfig}
\usepackage[usenames,dvipsnames]{xcolor}
\usepackage{dsfont} 
\usepackage{algorithm}        
\usepackage[noend]{algpseudocode} 
\usepackage{float}            
\usepackage{multirow}
\usepackage{booktabs}

\newcommand\ER{\mathcal{G}}
\newcommand{\donya}[1]{\textcolor{violet}{\textbf{[Donya:} #1\textbf{]}}}
\newcommand{\cyrus}[1]{\textcolor{blue}{\textbf{[Cyrus:} #1\textbf{]}}}

\newcommand{\donyaa}[1]{\textcolor{teal}{\textbf{[clean:} #1\textbf{]}}} 
\newcommand{\daniel}[1]{\textcolor{red}{\textbf{[Daniel:} #1\textbf{]}}}

\renewcommand{\donyaa}[1]{}
\renewcommand{\donya}[1]{}
\renewcommand{\cyrus}[1]{}
\renewcommand{\daniel}[1]{}

\DeclareFontFamily{U}{mathx}{\hyphenchar\font45}
\DeclareFontShape{U}{mathx}{m}{n}{
      <5> <6> <7> <8> <9> <10>
      <10.95> <12> <14.4> <17.28> <20.74> <24.88>
      mathx10
      }{}
\DeclareSymbolFont{mathx}{U}{mathx}{m}{n}
\DeclareMathSymbol{\bigtimes}{1}{mathx}{"91}
\usepackage{tcolorbox}
\tcbuselibrary{skins,breakable}
\tcbset{enhanced jigsaw}

\usepackage[normalem]{ulem}
\usepackage[compact]{titlesec}

\definecolor{DarkRed}{rgb}{0.1,0.1,0.8}
\definecolor{DarkBlue}{rgb}{0.1,0.1,0.5}

\usepackage{nameref}
\definecolor{ForestGreen}{rgb}{0.1333,0.5451,0.1333}
\definecolor{DarkRed}{rgb}{0.8,0,0.4}
\definecolor{Red}{rgb}{0.8,0,0.4}
\usepackage[linktocpage=true,
	pagebackref=true,colorlinks,
		urlcolor=DarkRed,
	linkcolor=DarkRed, citecolor=ForestGreen,
	bookmarks,bookmarksopen,bookmarksnumbered]
	{hyperref}
\usepackage[noabbrev,nameinlink]{cleveref}
\crefname{property}{property}{Property}
\creflabelformat{property}{(#1)#2#3}
\crefname{equation}{eq}{Eq}
\creflabelformat{equation}{(#1)#2#3}

\usepackage{bm}
\usepackage{url}
\usepackage{xspace}
\usepackage[mathscr]{euscript}

\usepackage{tikz}
\usetikzlibrary{arrows}
\usetikzlibrary{arrows.meta}
\usetikzlibrary{shapes}
\usetikzlibrary{backgrounds}
\usetikzlibrary{positioning}
\usetikzlibrary{decorations.markings}
\usetikzlibrary{patterns}
\usetikzlibrary{calc}
\usetikzlibrary{fit}
\usetikzlibrary{decorations.pathmorphing}

\usepackage{mdframed}

\makeatletter
\def\BState{\State\hskip-\ALG@thistlm}
\makeatother

\usepackage{cite}
\usepackage{enumitem}

\usepackage[margin=0.75in]{geometry}

\usepackage{natbib}
\setcitestyle{citesep={; }}

\newtheorem{lemma}{Lemma}[section]

\newtheorem{proposition}[lemma]{Proposition}
\newtheorem{theorem}[lemma]{Theorem}
\newtheorem{corollary}[lemma]{Corollary}
\newtheorem{claim}[lemma]{Claim}

\newtheorem{definition}[lemma]{Definition}

\newtheorem*{claim*}{Claim}
\newtheorem*{proposition*}{Proposition}
\newtheorem*{lemma*}{Lemma}
\newtheorem*{problem*}{Problem}
\crefname{lemma}{Lemma}{Lemmas}
\crefname{note}{Note}{Notes}
\crefname{claim}{Claim}{Claims}
\newtheorem{remark}[lemma]{Remark}

\newtheorem{mdresult}{Result}

\newtheoremstyle{restate}{}{}{\itshape}{}{\bfseries}{~(restated).}{.5em}{\thmnote{#3}}
\theoremstyle{restate}

\theoremstyle{definition}
\newtheorem{mdalg}{Algorithm}

\allowdisplaybreaks

\renewcommand{\qed}{\nobreak \ifvmode \relax \else
      \ifdim\lastskip<1.5em \hskip-\lastskip
      \hskip1.5em plus0em minus0.5em \fi \nobreak
      \vrule height0.75em width0.5em depth0.25em\fi}

\Crefformat{equation}{#2(#1)#3}
\crefformat{equation}{#2(#1)#3}
\Crefrangeformat{equation}{#3(#1)#4--#5(#2)#6}
\crefrangeformat{equation}{#3(#1)#4--#5(#2)#6}
\Crefmultiformat{equation}{#2(#1)#3}{ and~#2(#1)#3}{, #2(#1)#3}{ and~#2(#1)#3}
\crefmultiformat{equation}{#2(#1)#3}{ and~#2(#1)#3}{, #2(#1)#3}{ and~#2(#1)#3}

\mathtoolsset{centercolon}
\DeclarePairedDelimiterXPP\ind[1]{\mathds{1}}{\lbrace}{\rbrace}{}{#1} 
\DeclarePairedDelimiterX\eval[1]{\lbrace}{\rvert}{#1 \delimsize\rbrace} 
\DeclarePairedDelimiter\abs{\lvert}{\rvert} 
\DeclarePairedDelimiter\card{\lvert}{\rvert} 
\DeclarePairedDelimiter\norm{\lVert}{\rVert} 
\DeclarePairedDelimiter\ceil{\lceil}{\rceil} 
\DeclarePairedDelimiter\del{\lparen}{\rparen} 
\DeclarePairedDelimiter\set{\lbrace}{\rbrace} 
\DeclarePairedDelimiter\intoo{\lparen}{\rparen} 
\DeclarePairedDelimiter\intcc{\lbrack}{\rbrack} 

\DeclarePairedDelimiterX\Set[1]\lbrace\rbrace{%
  
  #1
}
\DeclarePairedDelimiterX\Sbr[1]\lbrack\rbrack{%
  
  #1
}
\DeclarePairedDelimiterX\Del[1]\lparen\rparen{%
  
  #1
}

\newcommand\T{{\scriptscriptstyle{\mathsf{T}}}} 

\def\ddefloop#1{\ifx\ddefloop#1\else\ddef{#1}\expandafter\ddefloop\fi}
\def\ddef#1{\expandafter\def\csname bf#1\endcsname{\ensuremath{\mathbf{#1}}}}
\ddefloop abcdefghijklmnopqrstuvwxyzABCDEFGHIJKLMNOPQRSTUVWXYZ\ddefloop
\def\ddef#1{\expandafter\def\csname bf#1\endcsname{\ensuremath{\boldsymbol{\csname #1\endcsname}}}}
\ddefloop {alpha}{beta}{gamma}{delta}{epsilon}{varepsilon}{zeta}{eta}{theta}{vartheta}{iota}{kappa}{lambda}{mu}{nu}{xi}{pi}{varpi}{rho}{varrho}{sigma}{varsigma}{tau}{upsilon}{phi}{varphi}{chi}{psi}{omega}{Gamma}{Delta}{Theta}{Lambda}{Xi}{Pi}{Sigma}{varSigma}{Upsilon}{Phi}{Psi}{Omega}{ell}\ddefloop
\def\ddef#1{\expandafter\def\csname cal#1\endcsname{\ensuremath{\mathcal{#1}}}}
\ddefloop ABCDEFGHIJKLMNOPQRSTUVWXYZ\ddefloop

\DeclareMathOperator\sign{sign} 
\newcommand\R{\mathbb{R}} 

\DeclareMathOperator*{\Prob}{\ensuremath{\mathbb{P}}}
\renewcommand{\Pr}{\Prob}

\newenvironment{tbox}{\begin{tcolorbox}[
		enlarge top by=5pt,
		enlarge bottom by=5pt,
		 breakable,
		 boxsep=0pt,
                  left=4pt,
                  right=4pt,
                  top=10pt,
                  arc=0pt,
                  boxrule=1pt,toprule=1pt,
                  colback=white
                  ]
	}
{\end{tcolorbox}}

\newcommand{\II}{\ensuremath{\mathbb{I}}}

\newcommand{\mireal}[1][]{
  \ifx\relax#1\relax%
    \II(\mione \,; \mitwo)%
  \else%
    \II(\mione \,; \mitwo\mid #1)%
  \fi
}

\DeclareMathOperator\E{\mathbb{E}} 

\usepackage{algorithm}
\usepackage{algpseudocode}



\title{\vspace{-1cm}\textbf{Prior Knowledge Makes It Possible: From Sublinear Graph Algorithms to LLM Test-Time Methods}}

\author{
{\bf Avrim Blum}\thanks{Toyota Technological Institute at Chicago. \texttt{avrim@ttic.edu}} \quad
{\bf Daniel Hsu}\thanks{Columbia University. \texttt{djhsu@cs.columbia.edu}} \quad
{\bf Cyrus Rashtchian}\thanks{Google Research. \texttt{cyroid@google.com}} \quad
{\bf Donya Saless}\thanks{Toyota Technological Institute at Chicago. \texttt{donya@ttic.edu}}
}

\date{\today}

\begin{document}
\vspace{-3cm}
\maketitle
\vspace{-0.5cm}
{\renewcommand\thefootnote{}\footnotetext{Authors listed in alphabetical order.}%
}
\begin{abstract}
    Test-time augmentation, such as Retrieval-Augmented Generation (RAG) or tool use, critically depends on an interplay between a model's parametric knowledge and externally retrieved information. However, the theoretical underpinnings of this relationship remain poorly understood. Specifically, it is not clear how much pre-training knowledge is required to answer queries with a small number of augmentation steps, which is a desirable property in practice. To address this question, we formulate multi-step reasoning as an $s$-$t$ connectivity problem on a knowledge graph. We represent a model's pre-training parametric knowledge as a partial, potentially noisy subgraph.  We view augmentation as querying an oracle for true edges that augment the model's knowledge. Then, we characterize the necessary and sufficient number of augmentation steps for the model to generate an accurate answer given partial prior knowledge. One key result shows a phase transition: if the prior knowledge graph over $n$ vertices is disconnected into small components, then finding a path via augmentation is inefficient and requires $\Omega(\sqrt{n})$ queries. On the other hand, once the density of correct knowledge surpasses a threshold, forming a giant component, we can find paths with an expected constant number of queries.
\end{abstract}
\section{Introduction}

Generating accurate and helpful answers with Large Language Models (LLMs) often involves a combination of \emph{thinking} about the user query and \emph{retrieving} relevant information from external sources. For reasoning problems, the LLM can produce a chain-of-thought, analyzing intermediate sub-problems before arriving at a final solution~\cite{comanici2025gemini, mirtaheri2025let,yang2025qwen3,deepseekai2025deepseekr1incentivizingreasoningcapability}. For information seeking queries, the LLM can retrieve from databases or knowledge graphs to expand its grasp of new facts~\citep{gutierrez2025rag, lewis2020retrieval, min2019knowledge, zhou2025depth, vu2023freshllms}. A common thread for these scenarios is that the LLM needs to augment its pre-training knowledge with additional information (either self-generated or external) before being able to adequately answer a question~\citep{joren2024sufficient,su2024bright, wei2024measuring}. However, this interplay between parametric and user-provided or externally-retrieved contextual knowledge remains poorly understood.

We develop a graph-theoretic framework to study the ability of LLMs to solve multi-step problems. Using our abstract model, we can shed light on some fundamental questions. For instance, we explore how properties of the pre-training knowledge can facilitate or impede the ability to solve a multi-step problem. Then, as one of our results we show that when prior quality is above a certain threshold, a constant expected number of augmentation steps suffices, and below this threshold, any strategy requires superconstantly many queries using external information. Of course, representing an LLM's vast, nuanced parametric knowledge as an unweighted subgraph is a significant simplification. Nonetheless, our model provides a formal lens on a key principle: we support the conjecture that efficient retrieval-augmented generation (RAG) requires a richness of parametric knowledge~\citep{guu2020retrieval, pan2024unifying, wei2024long, xie2023adaptive, liu2025understanding}. In other words, if LLMs are to succeed in RAG tasks, then they need to both process language \emph{and} have a sufficient density of world knowledge from their pre-training data. Similarly, we provide further evidence that the best reasoning models require a deep knowledge of mathematical facts~\citep{ma2025benchmarking}. 

\begin{figure*}[ht]
\begin{center}
\centerline{\includegraphics[width=0.8\textwidth]{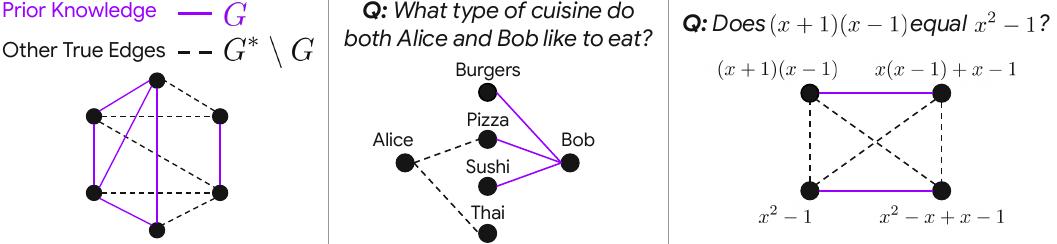}}
\caption{The left graph depicts that in our basic framework, we have a prior graph $G$, which is a subset of the true graph $G^*$. We illustrate two example knowledge graphs, where finding a path between nodes can answer the given question. Often the algorithm must query dotted edges from $G^* \setminus G$ that are outside of $G$. We also consider various ways to sample $G$ and $G^*$, as well as cases where edges in $G$ may be noisy and need verification.\vspace{-2em} }
\label{figure:teaser}
\end{center}
\end{figure*}
\subsection{Graph-theoretic Framework for Solving Multi-step Problems}
Given that an LLM with test-time augmentation is a very complex system to study, we focus on a few key components. At a high level, we consider a knowledge graph $G^*$ where nodes are entities in the world and edges represent relationships or facts that involve two entities. For example, consider a simple fact composition scenario in knowledge graphs: if we know ‘‘A is connected to B'' and ‘‘B is connected to C,'' we can deduce ‘‘A is connected to C.'' Such chaining of local facts into a global conclusion lies at the heart of deductive reasoning. One interpretation is that ``A is connected to B'' corresponds to ``A is equivalent to B'' and we want to deduce other equivalence relations. Another is that a connection corresponds to a co-occurrence in a sentence~\citep{yang2024synthetic}. We provide further examples in \Cref{figure:teaser}.

A core component of our framework is that we aim to formalize the difference between the partial knowledge that the model knows from pre-training and the potential facts that it could know through reasoning-based thinking or external retrieval mechanisms. We abstract this as designating a subgraph $G$ of $G^*$. In graph terms, we frame ``good'' prior knowledge as exhibiting well-connectedness, expansion, and edge reliability.
We use $G$ to capture the fact that a pre-tained LLM will have only partial information about an unknown ground-truth graph $G^*$. 
Furthermore, each edge of $G$ may or may not correspond to a true edge in $G^*$, and the number of edges in $G$ will be a small fraction of the total in $G^*$. We want to find trade-offs where a model knows $G$ but also requires some access to $G^*$ to correctly answer a question. The model will access $G^*$ using ``oracles'' that provide information. From a theoretical point of view, our work  complements much literature on sub-linear time graph algorithms~\citep{goldreich2011algorithmic}. In particular, the algorithm starts with the prior knowledge of $G$, rather than with no information about $G^*$, leading to a new twist on classical problems.

To study test-time augmentation, we define two query models: (i) given a vertex $v$, the \emph{retrieval oracle} returns a random neighbor of $v$ in $G^*$ (and our lower bound also holds for a stronger oracle that given a vertex $v$, returns \emph{all} neighbors in $G^*$ of the connected component in $G$ that contains $v$) and (ii) given a pair $u,v$, the \emph{verifier oracle} returns whether $\{u,v\}$ is an edge in $G^*$ or not. These oracles capture different aspects of test-time augmentations based on a knowledge graph. In a RAG setting, a retrieval mechanism provides information about the query. Often this is done through a similarity search (e.g., BM25 or embeddings) based on the entities in the query. However, we cannot guarantee exactly what information the retriever returns. Hence, in the retrieval oracle, we get a random neighbor entity of a vertex.
The connected component oracle is stronger, returning many neighbors. The verifier oracle offers the option to query a specific edge, rather than a random one. 

With these query models, we then study when it is possible to efficiently solve certain tasks. We are particularly interested in algorithms that use a constant number of queries, not scaling with graph size. We begin our analysis with the $s$–$t$ path problem. For an input pair of vertices $(s, t)$ in a hidden ground truth graph $G^*$, the goal is to output a path between $s$ and $t$ in $G^*$ if there is one. The path finding task is a proxy for a fundamental aspect of deductive reasoning: composing local facts (edges) to infer a global conclusion (path connectivity). For example, each edge may represent a discrete factual or logical step (e.g., `Alice likes Pizza' or `Pizza is an Italian food'). Finding a path from `Alice' to `Italian food' means composing individual facts to infer a new relationship. 
We also consider a robust version, where we output a subgraph that connects $s$ and $t$ even after removing certain edges.

\subsection{Main Results}
\begin{table*}[h!]
\centering
\caption{Algorithm Performance and Lower Bounds. Here a ``\checkmark'' for \emph{Grounded} means we provide a grounded algorithm (\Cref{def:GroundedAlgorithm}), and an ``$\times$'' means our lower bound holds for general algorithms. For brevity we only state results for  the retrieval oracle (\Cref{def:RetrievalOracle}) in this table, with results for other oracles in the paper.}
\label{tab:algo_results}
\begin{tabular}{lllcll}
\toprule
\textbf{Problem} & \textbf{Prior $G$} & \textbf{True $G^*$} & \textbf{Grounded} & \textbf{Results} & \textbf{Reference} \\
\midrule
\multirow{2}{*}{$s$–$t$ path} & Random dense subgraph & Erd\H{o}s–R\'enyi & \checkmark & $O(1)$ & Thm.~\ref{theorem:admissibleErods} \\
                           & Random sparse subgraph & Erd\H{o}s–R\'enyi & $\times$     & $\Omega(\sqrt{n})$ & Thm.~\ref{theorem:supernodelowerbound} \\
\midrule
$s$–$t$ path & Double Star & + Random Bridge & $\times$ & $\Theta(n)$ & Prop.~\ref{proposition:RandomizedAdaptive} \\
\midrule
$s$–$t$ path & Empty & Complete graph & \checkmark & $\Theta(\sqrt{n})$ & Prop.~\ref{Proposition:birthdayparadox}\\
\midrule
Int.~$K$-connected & Random dense subgraph & Erd\H{o}s–R\'enyi & \checkmark & $O(\log K)$ & Thm.~\ref{thm:disjoinpaths} \\
\bottomrule
\end{tabular}
\end{table*}

We prove several new bounds, for multiple query models and algorithmic tasks. \Cref{tab:algo_results} summarizes our main results. Our contributions make progress toward the larger goal of characterizing when an AI system, along with test-time mechanisms, can solve reasoning tasks from partial, and possibly noisy, prior knowledge.
We first introduce a graph-theoretic property, \emph{Retrieval Friendliness} (\Cref{def:RetrievalFriendliness}), that captures when a partial knowledge graph and its ground truth counterpart admit efficient reasoning about every $s$–$t$ connectivity prompt using a constant expected number of retrieval queries. 
We then define a general property, admissible, that implies Retrieval Friendliness. 
Building on these concepts, we show that variations of a bidirectional search algorithm can identify different types of subgraphs in $G^*$ while using few queries.

To instantiate our theory in a more concrete setting, we show that certain random graphs are admissible with high probability.
The Erd\H{o}s–R\'enyi graph model serves as a good testbed for reasoning from incomplete priors due to its homogeneity, connectedness, strong expansion, and small diameter. These are all characteristics that intuitively facilitate finding paths. 
In other words, if a constant-query algorithm already fails here, then we would need stronger assumptions on the prior knowledge (e.g., more true information or a property where connectivity is tied to locality).   

Our main technical results consist of new, nearly-tight lower bounds, which apply to multiple oracle models. First, we focus on the path problem with a random neighbor oracle, showing that can be hard to even find an $s$–$t$ path without sufficient prior knowledge.  We consider both worst-case and random graphs. Starting simple, we show that if $G$ misses a single ``bridge'' edge, then we  need $\Omega(n)$ queries.  We next analyze Erd\H{o}s–R\'enyi graphs, considering when the prior knowledge is \emph{random} as opposed to structured, where we get an $\Omega(\sqrt{n})$ lower bound.
Finally, we consider queries that return multiple neighbors, other tasks (e.g., multi-vertex connectivity), and robustness constraints.
We analyze the reliability of prior knowledge, quantifying how many ``false facts'' (incorrect edges) we can tolerate while verifying paths. 
This highlights the importance of verifying the model’s intermediate reasoning steps before relying on further retrieval. 

One salient aspect of our new lower bounds is that algorithms need many queries \emph{even when the expected path length is short}. It would be easy to prove that the number of queries grows with the path length. We go further, showing cases where the algorithm must explore many options, and the  difficulty comes from \emph{finding} a path. This is more interesting because LLMs often answer queries that only require a few hops.

\subsection{Related Work}
Retrieval-Augmented Generation (RAG) enhances LLMs by allowing them to access external knowledge bases. While RAG is effective in practice, the theoretical modeling of RAG is limited  \citep{koga2025privacypreservingretrievalaugmentedgenerationdifferential, weller2025theoretical},
and
the interplay between a model's existing knowledge and the information it retrieves
is not well understood. 
Classic RAG uses a single retrieve then generate step, which is often insufficient for multi-hop or evolving information needs. \emph{Dynamic RAG} interleaves generation with retrieval, deciding both \emph{when} and \emph{what} to retrieve \citep{DynamicRAG2025, asai2023self, gao2022rarr}. Corrective variants add evaluators to re-query when retrieval looks untrustworthy \citep{CRAG2024}. Our model complements these systems by replacing heuristic trigger policies with \emph{bounded query guarantees}: under structural conditions on the target knowledge graph, we characterize when constant expected retrieval suffice and when no bounded policy can succeed. Related instance-level criteria such as \emph{sufficient context} \citep{joren2024sufficient} evaluate whether the retrieved snippets alone contain a solution; our concept of \emph{retrieval friendliness} strengthens this by demanding constant-query, zero-error guarantees while considering the effect of prior knowledge. Our lower bounds provide more evidence for the theoretical limitations of embedding-based retrieval \citep{weller2025theoretical}.

Some of our results are inspired by 
a process-based supervision model \citep{uesato2022solvingmathwordproblems, lightman2023letsverifystepstep, setlur2024rewardingprogressscalingautomated, rohatgi2025tamingimperfectprocessverifiers, balcan2025learning}. Unlike outcome-only feedback, which evaluates a complete solution, process-based supervision provides granular feedback on each intermediate step. For graph problems, this corresponds to validating edges.
We model the step-level validation capability with a verifier oracle, which is a membership query \citep{10.1023/A:1022821128753} on the edge set of $G^*$. Graph-structured training and tool use have improved relational reasoning with LLMs \citep{mirtaheri2025let, yao2023treethoughtsdeliberateproblem, shalevshwartz2025reasoningsuperintelligencesearchtheoreticperspective, HRL22,huang2023prodigyenablingincontextlearning,wu2024avataroptimizingllmagents, kim2025metastabledynamicschainofthoughtreasoning}, and while synthetic continued pre-training \citep{yang2024synthetic} can strengthen the connectedness of parametric knowledge, our results clarify when prior knowledge can improve test-time retrieval efficiency.

Our work connects to a long line of literature on sublinear graph algorithms in query models; see \cite{beame2020edge, feige2004sums, feige2021tight, racz2019finding, rashtchian2020vector, rashtchian2021average, chen2020nearly} and references therein.  We utilize recent lower bounds on shortest paths in expanders and random graphs~\citep{alon2023sublinear}, where that research provides a foundation for studying path computations in large networks~\citep{basu2025sublinear}.  Our work is also related to the growing body of literature on {\em data-driven algorithm design} (e.g., \cite{balcan2017learning, balcan2020data, balcan2024much}) in that our graph $G$ can be viewed as a data sample from an underlying population $G^*$, and we are asking the question of how much data is sufficient to substantially improve the (query) complexity of a path-finding algorithm. Finally, our results imply lower bounds on CoT length in a reasoning-inspired model~\citep{mirtaheri2025let}.
\raggedbottom

\section{Preliminaries}
\label{sec:preliminaries}
For integer $n\in\mathbb{N}$, define $[n]=\{1,\dots,n\}$. Let $G^*=(V,E^*)$ be the ground truth graph on $V=[n]$. We use $N_{G}(u)$ for the neighbors of $u$ in a graph $G$, and let $\deg_G(u)=|N_G(u)|$. An \emph{$s$–$t$ path} is a simple path $P=(v_0=s, v_1,\ldots,v_k=t)$ consisting of distinct pairs $(v_i,v_{i+1})\in E^*$.
The algorithm has access to a prior graph $G=(V,E)$. To isolate the challenge of knowledge incompleteness, we begin by assuming the prior is reliable; that is, $G$ is a subgraph of $G^*$ with $E \subseteq E^*$. This `clean prior' setting allows us to first study knowledge structure, before we later discuss extensions to unreliable or ‘hallucinated' facts (incorrect edges). The algorithm can use one or more oracles to $G^*$, which serve as test-time augmentation methods.  In addition to paths, we will also be interested in ``robust'' subgraphs.
For $K\geq 1$, we say that $P$ is an \emph{internally $K$-connected} subgraph between $s$ and $t$ if $s$ and $t$ remain connected in $P$ whenever fewer than $K$ edges are removed from $P \cap G$.

\subsection{Retrieving Relevant Knowledge}
We start with our first query model. It is the most restrictive, but it will suffice for our algorithms. RAG systems provide relevant retrieval results through a variety of search methods. To abstract away their inner workings, we consider a basic \emph{Retrieval Oracle} that, given a vertex $u$, returns one of its true neighbors in $G^*$ chosen uniformly at random.

\begin{definition}[Retrieval Oracle]
    Let $G^*=(V,E^*)$ be a ground-truth graph. The retrieval oracle $\mathcal{O}_{G^*}:V\to E^*\cup\{\bot\}$ is specified by the family $\{\mu_u^{G^*}\}_{u\in V}$ where each $\mu_u^{G^*}$ is the uniform probability distribution over neighbors of $u$ in $G^*$. On query $u\in V$ the oracle returns
    \[
    \mathcal{O}_{G^*}(u)=
    \begin{cases}
    (u,v) & \text{with probability }\mu_u^{G^*}(v),\ v\in N_{G^*}(u),\\
    \bot  & \text{if } N_{G^*}(u)=\emptyset.
    \end{cases}
    \]
\label{def:RetrievalOracle}
\end{definition}
While real-world RAG systems use deterministic similarity searches, the retrieved results are imperfect proxies for true relevance. Modeling the output as stochastic is a tractable way to capture the uncertainty an algorithm faces, preventing it from exploiting an all-knowing retriever.
Later, we prove a lower bound for a stronger oracle that returns many neighbors based on the connected component containing the query.

We are interested in the interplay between the pretrained knowledge and the feasibility of outputting a correct answer. From an efficiency point of view, we also want to determine when a model can use a small number of retrieval queries in expectation. This is in contrast to cases where the model must make a number of queries that scales with the graph size, which would be infeasible in practice for large graphs. Our framework captures the fact that the the model often combines the knowledge learned through context with its own prior knowledge for answer generation.

We introduce the notions of \emph{Grounded Algorithms} and \emph{Retrieval Friendliness}. First, the distinction between grounded and general algorithms is crucial. A grounded algorithm should not ``hallucinate'' or guess connections; its reasoning is based on verified facts (e.g., via $G$ or an oracle). This captures how RAG systems  ideally work, where grounding refers to providing citations~\citep{gao2023enabling, songmeasuring}. Our lower bounds that hold for general algorithms (marked with an ``$\times$'' in \Cref{tab:algo_results}) are stronger and apply to hypothetical algorithms with the ability to guess edges.

\begin{definition}[Grounded Algorithm]
An algorithm $\mathcal{A}$ is \emph{grounded} if it outputs only edges it has explicitly observed from oracle outputs or prior knowledge $G$.
\label{def:GroundedAlgorithm}
\end{definition}

Combined with grounded algorithms, our next definition strengthens the notion of ``sufficient context'' of~\citet{joren2024sufficient}. 
Sufficient context means that the LLM has enough information to answer a query. We go further, saying that the algorithm can make a conclusion after a constant number of queries in expectation, and it only uses explicitly observed edges.
\begin{definition}[$q$-Retrieval Friendliness]\label{def:RF}
     The pair $(G,G^*)$ is \emph{$q$-retrieval friendly} for a grounded algorithm $\mathcal{A}$  if given $s,t$, access to $G$, and $\mathcal{O}_{G^*}$, the algorithm outputs:
    (a) $\textsc{no}$ when $t$ is not reachable from $s$ in $G^*$, and
    (b) when $t$ is reachable, a simple $s$-$t$ path $P$ such that all edges of $P$ are valid (they are also in $G^*$) by making at most $q$ queries in expectation. 
\label{def:RetrievalFriendliness}
\end{definition} 
Intuitively, Retrieval Friendliness implies that even though $G$ may contain incomplete information about the true graph $G^*$, it is still possible to efficiently recover valid reachability information for every pair in $G^*$ using only a constant number of retrieval queries in expectation. 
\subsection{Random Graphs \& Asymptotic Notation}
Our general framework applies to any way of constructing a pretrained knowledge graph $G$ and a target graph $G^*$. In some cases, we analyze a standard random graph model. Let  $\ER(n,p)$ denote the Erd\H{o}s–R\'enyi random graph with $n$ vertices, where edges appear independently with probability $p$ (which may depend on $n$, so we may write $p(n)$ for clarity).
This model produces a high-entropy graph, with no correlation between edges, and provides a challenging regime for finding paths.
All our asymptotics are as \( n \to \infty \). An event $B$ occurs \emph{with high probability} if \( \lim_{n \to \infty}  \mathbb{P}[B] \to 1\). 
Unless stated otherwise, success probabilities are over randomness in $G$, $G^*$, the oracle, and any internal randomness in the algorithm. 
\section{Lower Bounds}
We establish limits of test-time augmentation by proving query complexity lower bounds. We begin with an  adversarial ``bridge'' graph to show that even a single missing piece of information can be expensive to find. We then show that without any prior knowledge, grounded algorithms are inefficient even on well-connected graphs. Finally, we prove our main lower bound for the more complex setting of random graphs.

\textbf{Bridge Graph Lower Bound.} While dense priors can permit efficient retrieval, we now demonstrate that sheer volume of pretrained knowledge by itself is insufficient. 
To illustrate this, we construct a worst-case instance where the prior $G$ is a subgraph of the target $G^*$ containing all but a single edge
forming an information bottleneck (see Figure~\ref{figure:worstcase}). Formally, define the \textbf{double star with random bridge} on $n$ vertices as a graph $G^*=(V,E^*)$ constructed as follows: $V$ is partitioned into $S$ and $T$ of size $n/2$. Then, $E^*$ contains the edges of two stars centered at designated vertices $c_s \in S$ and $c_t \in T$, plus a single bridge edge $(u,v)$ where $u \in S\setminus\{c_s\}$ and $v \in T\setminus\{c_t\}$ are chosen uniformly at random. The learner knows the prior graph $G:=G^*\setminus\{(u,v)\}$. Let $(s,t)$ be a pair of vertices chosen uniformly at random from $V$. This  example demonstrates a lower bound in an extreme case. 

\begin{figure}[ht]
\begin{center}
\centerline{\includegraphics[width=0.25\columnwidth]{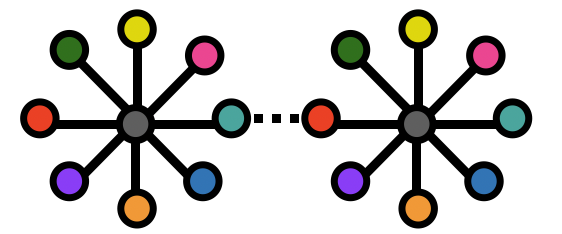}}
\caption{Double star with random bridge.  The prior knowledge graph $G$ consists of two disjoint star graphs on the left and right. The ground-truth graph $G^*$ adds a single, hidden ``bridge'' edge between random leaves on the left and right. Any algorithm must query $\Omega(n)$ leaves on average to find this bottleneck edge.
}
\label{figure:worstcase}
\end{center}
\vskip -0.3in
\end{figure}
\begin{proposition}
    Finding an $s$-$t$ path in the double star with random bridge on $n$ vertices requires $\Omega(n)$ retrieval queries to have success probability $\geq 2/3$. \label{proposition:RandomizedAdaptive}
\end{proposition}
Intuitively, the algorithm must query vertices until it finds the bridge. However, each query reveals no extra information about the potential bridge endpoints, since the retrieval oracle returns a random neighbor. We provide the full proof in \Cref{appendix:RandomizedAdaptive}.  Note that the proof holds with access to a stronger retrieval oracle that treats every edge in the prior graph as already known and never repeats an edge it has either previously returned or that was present in the prior.

\textbf{Complete Graph, Grounded Lower Bound.} We next establish that in a favorable setting, constant query \emph{grounded} retrieval is impossible without prior knowledge. We consider when the ground truth graph $G^*$
is a unweighted complete graph, and the pretrained graph $G$ has no edges. This is a worst case scenario for the learner in terms of prior information, but best-case in terms of the target graph's connectivity.

\begin{proposition}
    Let $G^*$ be complete and $G$ empty. For any nodes 
    $s$ and $t$, any grounded algorithm must make $\Omega(\sqrt{n})$ retrieval oracle (\Cref{def:RetrievalOracle}) queries to find an $s$-$t$ path with success probability at least $1/2$.
    \label{Proposition:birthdayparadox}
\end{proposition}
    The proof is an application of the birthday paradox and is stated in \Cref{birthdayparadox-lowererdos}. 
    We note that the proof holds with access to a stronger retrieval oracle that never repeats an edge it has previously returned.
This lower bound holds for the fundamental problem of finding any path between two nodes, not merely a shortest path. 
The core of our argument is based on collision probability (i.e., the birthday paradox) and would apply to $G^* \sim \ER(n,p)$ for any value of $p$.

\textbf{General Lower Bound \& Stronger Oracle.} We now extend this result to the more general setting of Erd\H{o}s–R\'enyi random graphs, even when the learner has access to a powerful retrieval oracle that given a vertex $v$ returns \emph{all} neighbors in $G^*$ of the connected component in pretrained graph $G$ that contains $v$.

\begin{definition}[Connected Component Incident
Retrieval Oracle]
    Let $G^*$ be the ground truth graph and $G$ the pretrained subgraph. For any vertex $v \in V$, let $C_G(v) \subseteq V$ denote the set of vertices in the connected component of $G$ that contains $v$.
    The \emph{CCI retrieval oracle} is a map
    $
    \mathcal{O}^{\mathrm{cci}}_{G^*, G}: V \to 2^{E^*} \cup \{\bot\}
    $
    defined by
    \begin{align*}
        \mathcal{O}^{\mathrm{cci}}_{G^*,G}(v)=
        \begin{cases}
        S_v, & S_v\neq\emptyset,\\
        \bot, & \text{otherwise,}
        \end{cases}
    \end{align*}
    where $S_v:=\{(u,w)\in E^*: u\in C_G(v),\ w\notin C_G(v)\}$.\vspace{-0.3em} 
\label{def:ccincidentquery}
\end{definition}

Consider $G^*\sim \ER (n,p)$ with $p \ge \tfrac{1.5\log n}{n}$, which ensures $G^*$ is connected with high probability. 
We are interested in the sparse but supercritical regime when $p$ is above the connectivity threshold yet far from dense. For instance, when $p=1$ then $G^*$ is the complete graph and a single CCI query (\Cref{def:ccincidentquery}) trivializes path finding. 
Assume the learner also has access to pretrained knowledge $G$ obtained by independently retaining each edge of $G^*$ with probability $\eta$ with $p\cdot\eta < \frac{1}{n}$, placing $G$ in a regime that carries negligible {\em global} connectivity signal.

\begin{theorem}
    In the setup above, any 
    algorithm for finding a path between given vertices $s,t$ 
    either makes \(\Omega(\frac{1}{p \cdot \log^2 n \cdot \sqrt{n}})\) connected component incident retrieval (\Cref{def:ccincidentquery}) queries or finds an $s$–$t$ path with probability at most $p \log^2 n +o(1)$.
    \label{theorem:supernodelowerbound}
\end{theorem}
The proof is provided in \Cref{erdosquery} and relies on a reduction that contracts $O(\log n)$-sized components into super-nodes, turning the problem into path finding in a meta-graph. It then uses a trace based analysis, that is, by fixing the algorithm’s randomness, each execution is represented as a trace.
Observe the number of edges discovered after \(\Omega(\frac{1}{p \cdot \log^2 n \cdot \sqrt{n}})\) connected component incidence queries is \(O(\sqrt{n})\) each call can reveal fewer than \(p\cdot\log^2 n \cdot (n-1)\) incident edges.  
Having shown the $\Omega(\sqrt{n})$ lower bound for path recovery, we now present a nearly matching upper bound for Erd\H{o}s-R\'{e}nyi random graphs, which adapts a result of \cite{alon2023sublinear} for regular expander graphs to random graphs that are only approximately regular.
\begin{theorem}
    Consider an Erd\H{o}s-R\'{e}nyi random graph $G^* \sim G(n,p)$, where $p (1-p) \geq C \cdot \frac{\log^4(n)}{n}$ and $C$ is a sufficiently large constant. There exists an algorithm that, with high probability over the randomness of $G^*$,  
    for every node $s$ and every $\delta \in (0,1)$, finds an $s$–$t$ path while visiting $O((\frac{n}{\delta})^{\frac{1}{2} + o(1)})$ vertices for all but a $\delta$-fraction of targets $t$.
\label{RMT}
\end{theorem}

The proof in \Cref{Proof:RMT} uses tools from random matrix theory and is implementable with access to a retrieval oracle that never repeats an edge it has previously returned.
Next, we explore the properties of the ground truth and prior knowledge graph that make retrieval friendliness (\Cref{def:RetrievalFriendliness}) possible.
\section{Upper Bounds: Efficient Test-Time Augmentation Algorithms}
\subsection{Reliable and Sufficient Prior for Efficient Retrieval}
We start by stating a condition that makes retrieval friendliness possible by the strategy of decomposing the task of finding paths into {\em efficient sub-tasks}. 

Intuitively, for path-finding, the model's prior knowledge $G$ must connect disparate regions of the complete knowledge graph $G^*$. Even if we do not know how to globally connect our start $s$ and end $t$ nodes, we can find a subgraph that connects to both. Our formal condition, which we call admissibility, captures this idea: a well-connected component in the prior knowledge is ``visible'' from every node in the graph, covering a constant fraction of its local connections. We formalize this below and then show how a simple bidirectional search strategy can exploit it for efficient retrieval.

\begin{definition}[$\gamma$-Admissible Pair]
Let $G^*$ be the ground-truth graph, $G$ the pretrained subgraph, $\mathcal{O}_{G^*}$ the retrieval oracle,
and $\gamma \in (0,1]$.
For every vertex $u$, let $\mu_u^{G^*}$ be the uniform probability distribution over $N_{G^*}(u)$ given by $G^*$.
We say that $(G^*, G)$ is \emph{$\gamma$-admissible} if there exists a connected component $C$ of $G$ such that, for every vertex $u$ with $N_{G^*}(u) \neq \emptyset$,
\begin{align*}
    \mu_u^{G^*}\big( N_{G^*}(u) \cap V(C) \big) \ge \gamma.
\end{align*}
for some constant $\gamma$.
\label{def:admissiblepair}
\end{definition}

\begin{algorithm}[h]
\caption{Bidirectional-Retrieval Augmentation Generation (BiRAG)}
\begin{algorithmic}[1]
    \Require $\gamma$-admissible pair $(G^*,G)$, endpoints \(s,t\), retrieval oracle \(\mathcal{O}_{G^*}\).
     \State \(E_s \gets \emptyset,\; E_t \gets \emptyset\)
    \Repeat
        \State $e_s \gets \mathcal{O}_{G^*}(s)$
        \State $e_t \gets \mathcal{O}_{G^*}(t)$
        \If{$e_s=\bot$ \textbf{or} $e_t=\bot$} \State \Return \textsc{NO} \EndIf
        \State $E_s \gets E_s \cup \{e_s\}$ \State $E_t \gets E_t \cup \{e_t\}$
        \State \textbf{Augment:} \(\tilde G \gets \textsc{Augment}(G, E_s, E_t)\)
        \State \textbf{Generate:} $\Pi \gets$ $s$–$t$ path from $\tilde G$
    \Until{$\Pi \neq \bot$}
    \State \Return $\Pi$
\end{algorithmic}
\label{algorithm:giantcomponent}
\end{algorithm}

\begin{claim}
    Any $\gamma$-admissible pair is $2/\gamma$-retrieval friendly for bidirectional-retrieval augmentation generation algorithm (\Cref{algorithm:giantcomponent}).
    \label{claim:admissible}
\end{claim}

The bidirectional strategy is given in \Cref{algorithm:giantcomponent}.  After each pair of retrievals we \emph{augment} the pretrained graph with the returned edges and attempt to generate an $s$-$t$ path in the augmented graph.  Concretely, the Augment step forms $\tilde G = (V,\; E \cup E_s \cup E_t)$
by adding the edges in \(E_s\) and \(E_t\) to \(G\).  If the Generate step fails to find an $s$-$t$ path in \(\tilde G\), it returns \(\Pi=\bot\) and the loop continues.  Under the \(\gamma\)-admissibility assumption the loop in \Cref{algorithm:giantcomponent} runs $1/\gamma$ times in expectation, issuing two retrieval calls per iteration. We provide the run time analysis and the full proof in \Cref{proof:admissiblepair}.

As an example, we now show that our condition for efficient retrieval holds with high probability in an Erdős-Rényi graph model $G^*$, provided it contains a sufficiently dense subgraph $G$.

\begin{theorem}
\label{theorem:admissibleErods}
    Consider an Erd\H{o}s–R\'enyi random graph $G^* \sim \ER(n,p)$ with $p > C_0 \frac{\log n}{n}$ for a sufficiently large constant $C_0$. Let pretrained graph $G$ be a subgraph formed by retaining each edge of $G^*$ independently with probability $\eta \in (\frac{1}{\log n}, 1]$. Then, with high probability\footnote{with probability $1-o(n^{-2})$} over the randomness of $G^*$ and $G$, the pair $(G,G^*)$ is $\gamma/3$-admissible with $\gamma \in (0,1]$ being the unique solution to $\gamma=1-e^{-np\eta\gamma}$.
\end{theorem}

As long as the edge probabilities $p$ (for the true graph) and $\eta$ (for the prior) are above the connectivity threshold, a giant component in the prior graph exists with high probability.  Furthermore, the random nature of the remaining edges in $G^* \setminus G$ ensures that this component is distributed, making it `visible' from all other nodes, thus satisfying our admissibility condition. 
Formally, using the equivalence that $G^*$ can be obtained by adding independent edges to $G$, we can condition on the giant component of $G$ to show every vertex connects into it with ratio at least $\gamma/3$.
Note that since $np\eta > C_0$ in \Cref{theorem:admissibleErods},  $\gamma$ is always at least some absolute positive constant.
Full details are in  \Cref{proof:admissibleErdos}.
\begin{figure}[ht]
\begin{center}
\centerline{\includegraphics[width=0.35\columnwidth]{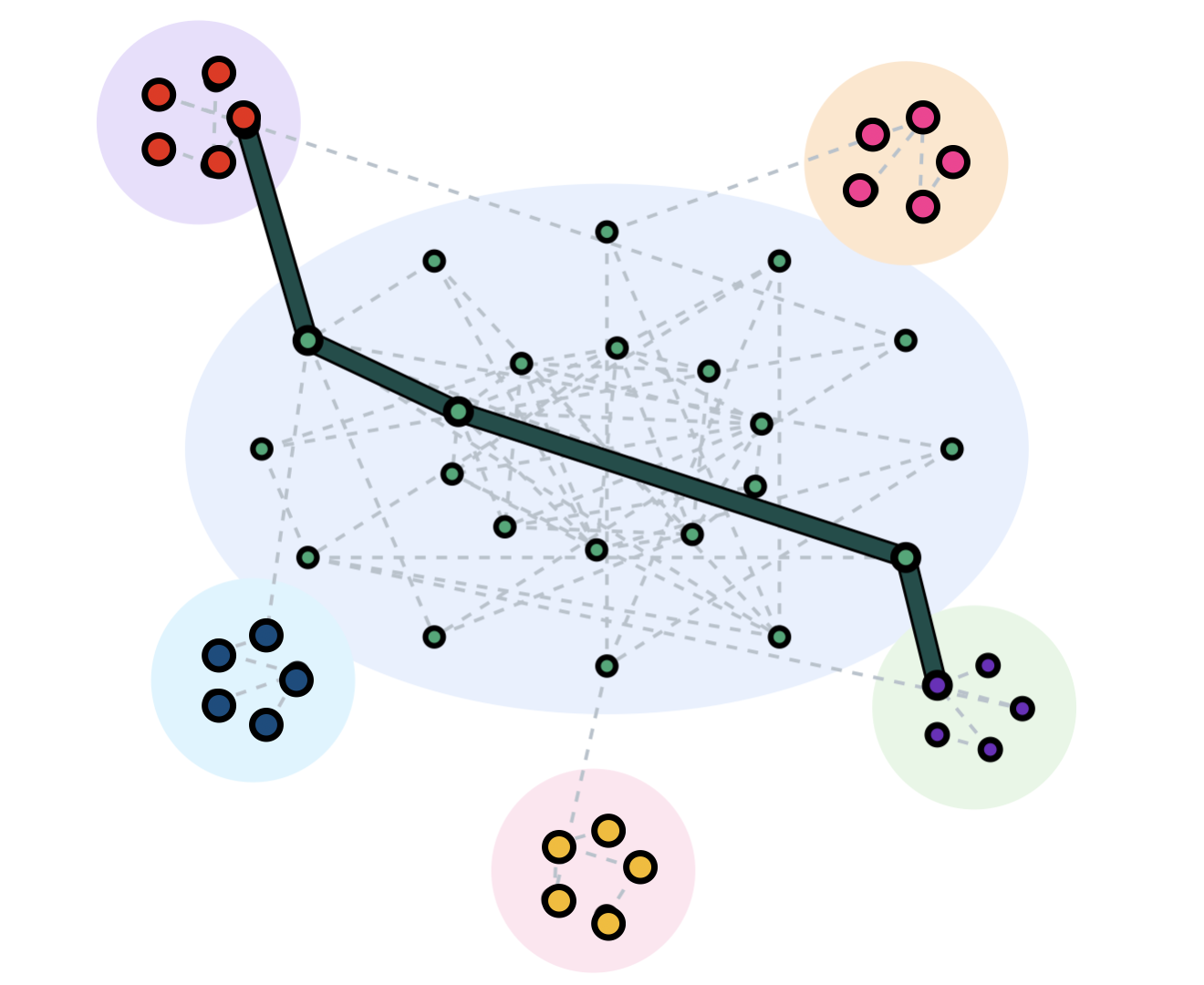}}
\caption{Illustrating $\gamma$-admissible \& \Cref{algorithm:giantcomponent}.}
\end{center}
\vskip -0.3in
\end{figure}
Motivated by priors with community structure, we consider a partitioned revelation model that preserves intra-group edges while suppressing cross-group links.
\begin{remark}
    In \Cref{theorem:admissibleErods}, suppose the pretrained graph is obtained through a much harsher non-uniform revelation process. Let $V=\biguplus_{l=1}^L V_l$ be a fixed partitioning of the vertices for some constant $L$, and the pretrained graph $G$ be a subgraph formed by retaining each \emph{intra-group} edge of $G^*$ independently with probability $\eta \in (\frac{1}{\log n}\cdot\frac{n}{\,\max_l |V_l|},1]$, and discarding all inter-group edges (i.e., $p_{i,j}=\eta\cdot \mathbf{1}\{\exists l: i,j\in V_l\}$). Then, with high probability over the randomness of $G^*$ and $G$, the pair $(G,G^*)$ is admissible.
    \label{remark:Kpartitions}
\end{remark}
\paragraph{Beyond Paths: Finding Steiner Trees.}
One natural extension to finding a path between vertices is considering a set of $M$ input vertices $(s_1,\ldots,s_M)$ and asking whether the learner can recover a Steiner tree connecting them all. This is a proxy for finding a set of facts that connect many entities, e.g., when an LLM must construct a coherent story involving multiple people. We can extend our bidirectional search to an $M$-directional algorithm, which works for admissible graphs. From each $s_i$, we connect it to the giant component and then stitch the paths together. That is, in $\gamma$-admissible pairs, we can find a Steiner tree containing the neighborhoods of $(s_1,\ldots,s_M)$ inside the giant component of $G$, and connect the $s_i$ to it with $M$ extra edges and by making $M/\gamma$ queries in expectations.

\subsection{Sufficient but Unreliable Prior}
\label{verification}
To establish a robust conclusion between two entities, one should not rely on a single line of reasoning. We therefore seek \emph{many} candidate routes whose evidence is separated across the pretrained graph. Concretely, think of an $K$ edge-coloring (or labeling) of $E$: each color (labeling) is a self-contained reasoning route. Concretely, we want $K$ edge-disjoint segments in the pretrained graph.
The benefit is robustness: even if we distrust up to $K-1$ edge types, a single remaining route of a trusted type still suffices to certify correctness. The next definition formalizes an structural property that satisfies this. 
\begin{definition}[$(K, \gamma)$-Robust-Admissible]\label{def:layerwise-global}
    Let $G^*$ be the ground truth graph, $G$ the pretrained subgraph, 
    $K\in\mathbb{N}$ and $\gamma\in(0,1]$. For each vertex $u$, let $\mu_u^{G^*}$ be the uniform probability distribution over $N_{G^*}$ given by $G^*$.
    We say $(G^*,G)$ is \emph{$(K, \gamma)$-robust-admissible} if there exist
    an edge-partition $E=\biguplus_{k=1}^K E_k$, with $G_k:=(V,E_k)$, and for each $i\in[k]$, a connected component $C_k$ of $G_k$ such that $\text{$\forall u\in V$ with }N_{G^*}(u)\neq\emptyset\text{ and all } k\in[K]$
    \begin{align*}
        \mu_u^{G^*}\big( N_{G^*}(u)\cap V(C_k)\big) \ \ge\ \gamma.
    \end{align*}
\end{definition}
We now present an example of such a pair.
\begin{corollary}
    Consider $G^*\sim \ER(n,p)$ with $p > C_1 \frac{\log n}{n}$, where $C_1 = K C_0$, and $C_0$ is the same constant as \Cref{theorem:admissibleErods}.  Let the pretrained graph $G$ be obtained by retaining each edge of $G^*$ independently with probability $\eta \in (\frac{1}{\log n}, 1]$.
    Then, with high probability over the randomness of $G^*$ and $G$, $(G^*,G)$ is $(K,\gamma/3)$-robust-admissible with $\gamma \in (0,1]$ the unique constant solution to $\gamma=1-e^{-np\eta\gamma}$. 
\end{corollary}
\begin{proof} 
 Independently color each retained edge of $G$ uniformly with one of $K$ colors, yielding an edge-partition $E=\biguplus_{k=1}^K E_k$ and subgraphs $G_k:=(V,E_k)$.  For each $k$, an edge of $G^*$ lands in $E_k$ with probability $\eta/K$, so marginally $G_k \sim\ \ER(n,\; \frac{\eta p}{K}).$  As a direct corollary of \Cref{theorem:admissibleErods}, for each $k\in [K]$ with high probability over the randomness of $G^*$, $G$ and $G_k$, the pair $(G_k,G^*)$ is $\gamma/3$-admissible. The union bound gives that this holds for all $k\in[K]$ simultaneously with high probability. Therefore, it follows with high probability over the randomness of $G^*$ and $G$, the pair is $(K,\gamma/3)$-robust-admissible.   
\end{proof}
\begin{theorem}
\label{thm:disjoinpaths}
    There exists an algorithm that for any $(K, \gamma)$-robust-admissible pair $(G^*, G)$, for any two vertices $s, t \in V$ makes $O(\frac{\log K}{\gamma})$ calls to the retrieval oracle (\Cref{def:RetrievalOracle}) in expectation and
    constructs internally $K$-connected subgraph between $s$ and $t$.
\end{theorem}
While the statement above is similar to a coupon collector argument we note that here every partition is hit with probability at least $\gamma$; hence, $O(\frac{\log K}{\gamma})$ queries suffice. The proof is provided in \Cref{proof:coupons}. 

\begin{remark}
    There exists an algorithm that for any $(K, \gamma)$-robust-admissible pair $(G^*, G)$, for any two vertices $s, t \in V$ makes $O(1/\gamma)$ calls to the retrieval oracle (\Cref{def:RetrievalOracle}) in expectation and
    constructs internally $K/2$-connected subgraph between $s$ and $t$.
\end{remark}

\paragraph{From Robustness to Verification}
The use of large pretrained models as priors is promising, but their inherent unreliability creates a fundamental trade-off. In the discussed robustness guarantee above, we assumed $E$ can be partitioned into $k$ types and each type sees a constant fraction of true neighbors. Thus, trusting any one edge type yields a correct $s,t$ path with efficient retrieval. When this assumption may fail we must switch to \emph{verification}. That is, use the prior $G$ to generate candidate $s,t$ paths and let a verifier certify the first valid one. Concretely, a verifier for a graph $G^* = (V,E^*)$ is an oracle that given two vertices $(u,v)$, returns $\textsc{yes}$ if $(u,v) \in E^*$ and $\textsc{no}$ otherwise.
This approach does not need a strong structural assumption and is appropriate when a verifier is available. To get a sense of how efficient the \emph{generate then verify} approach can be, imagine that the prior graph is reasonably {\em grounded}: whenever it suggests a short path of length $c$, each edge in that path has at least probability $r$ of being correct in the true graph. In other words, a whole $c$-hop path from the prior survives in the ground truth graph with probability at least $r^c$. Under this assumption, it is straightforward to see that we can find and certify a path using only about $O(c/r^c)$ verifier queries in expectation. 

\section{Conclusion}
We provide the first theoretical framework of test-time augmentation with multiple types of query models. We analyze the $s$–$t$ path finding problem serving as a basic testbed for reasoning
as well as the internally $K$-connected problem, a generalization of it. We study the interplay between the test time query complexity of solving these problems and prior knowledge in various types of graphs, providing upper and lower bounds. Depending on properties of prior knowledge, our bounds delineate ``easy'' regimes where we can find a correct $s$–$t$ path with few retrieval query calls as well as ``hard'' regimes where any algorithm requires $\Omega(\sqrt{n})$ or even $\Omega(n)$ queries. 
Overall, our results provide evidence that density and the structure of the pretrained knowledge is critical for efficient RAG or tool use. We list several problems in~\Cref{FutureDirections}.

\textbf{Limitations.} One shortcoming is that we only study a few oracle models, and there may be different trade-offs for other test-time augmentation methods. For example, it would be ideal to more closely align with similarity-based retrieval methods in real RAG systems. Another limitation is that our asymptotic analysis may not be precise enough to explain the nuanced trade-offs in real-world systems. As another aspect, our upper-bound results address the existence of certain subgraphs rather than the optimal versions (e.g., shortest path or minimum spanning tree), which we view as an important direction for future work. Additionally, while we studied multiple, distinct graph families, we did not fully characterize all ways to generate the prior knowledge graph $G$ and the target graph $G^*$.
Finally, retrieval friendliness is a broad concept, extending well beyond the path-finding problem. Characterizing the algorithms and conditions that enable it across problems remains a compelling direction for theory and practice.

\section*{Acknowledgments}
This work was supported in part by the National Science Foundation under grants CCF-2212968, DMS-2502259, and ECCS-2216899, by the Simons Foundation under the Simons Collaboration on the Theory of Algorithmic Fairness, and by the Office of Naval Research under MURI Grant N000142412742 and Grant N000142412700. We would like to thank Parsa Mirtaheri, Enric Boix-Adserà, Andrew Tomkins, and Rocco Servedio for helpful discussions.
\bibliography{general}

@inproceedings{songmeasuring,
  title={Measuring and Enhancing Trustworthiness of LLMs in RAG through Grounded Attributions and Learning to Refuse},
  author={Song, Maojia and Sim, Shang Hong and Bhardwaj, Rishabh and Chieu, Hai Leong and Majumder, Navonil and Poria, Soujanya},
  booktitle={The Thirteenth International Conference on Learning Representations},
  year={2025}
}

@inproceedings{gao2023enabling,
  title={Enabling Large Language Models to Generate Text with Citations},
  author={Gao, Tianyu and Yen, Howard and Yu, Jiatong and Chen, Danqi},
  booktitle={Proceedings of the 2023 Conference on Empirical Methods in Natural Language Processing},
  pages={6465--6488},
  year={2023}
}

@article{beame2020edge,
  title={Edge estimation with independent set oracles},
  author={Beame, Paul and Har-Peled, Sariel and Ramamoorthy, Sivaramakrishnan Natarajan and Rashtchian, Cyrus and Sinha, Makrand},
  journal={ACM Transactions on Algorithms (TALG)},
  volume={16},
  number={4},
  pages={1--27},
  year={2020},
  publisher={ACM New York, NY, USA}
}

@article{weller2025theoretical,
  title={On the Theoretical Limitations of Embedding-Based Retrieval},
  author={Weller, Orion and Boratko, Michael and Naim, Iftekhar and Lee, Jinhyuk},
  journal={arXiv preprint arXiv:2508.21038},
  year={2025}
}

@article{wei2024measuring,
  title={Measuring short-form factuality in large language models},
  author={Wei, Jason and Karina, Nguyen and Chung, Hyung Won and Jiao, Yunxin Joy and Papay, Spencer and Glaese, Amelia and Schulman, John and Fedus, William},
  journal={arXiv preprint arXiv:2411.04368},
  year={2024}
}

@article{vu2023freshllms,
  title={Freshllms: Refreshing large language models with search engine augmentation},
  author={Vu, Tu and Iyyer, Mohit and Wang, Xuezhi and Constant, Noah and Wei, Jerry and Wei, Jason and Tar, Chris and Sung, Yun-Hsuan and Zhou, Denny and Le, Quoc and others},
  journal={arXiv preprint arXiv:2310.03214},
  year={2023}
}

@article{gao2022rarr,
  title={Rarr: Researching and revising what language models say, using language models},
  author={Gao, Luyu and Dai, Zhuyun and Pasupat, Panupong and Chen, Anthony and Chaganty, Arun Tejasvi and Fan, Yicheng and Zhao, Vincent Y and Lao, Ni and Lee, Hongrae and Juan, Da-Cheng and others},
  journal={arXiv preprint arXiv:2210.08726},
  year={2022}
}

@inproceedings{feige2004sums,
  title={On sums of independent random variables with unbounded variance, and estimating the average degree in a graph},
  author={Feige, Uriel},
  booktitle={Proceedings of the thirty-sixth annual ACM symposium on Theory of computing},
  pages={594--603},
  year={2004}
}

@article{racz2019finding,
  title={Finding a planted clique by adaptive probing},
  author={R{\'a}cz, Mikl{\'o}s Z and Schiffer, Benjamin},
  journal={arXiv preprint arXiv:1903.12050},
  year={2019}
}

@article{feige2021tight,
  title={A tight bound for the clique query problem in two rounds},
  author={Feige, Uriel and Ferster, Tom},
  journal={arXiv preprint arXiv:2112.06072},
  year={2021}
}

@article{rashtchian2020vector,
  title={Vector-matrix-vector queries for solving linear algebra, statistics, and graph problems},
  author={Rashtchian, Cyrus and Woodruff, David P and Zhu, Hanlin},
  journal={arXiv preprint arXiv:2006.14015},
  year={2020}
}

@article{goldreich2011algorithmic,
  title={Algorithmic aspects of property testing in the dense graphs model},
  author={Goldreich, Oded and Ron, Dana},
  journal={SIAM Journal on Computing},
  volume={40},
  number={2},
  pages={376--445},
  year={2011},
  publisher={SIAM}
}

@inproceedings{basu2025sublinear,
author = {Basu, Sabyasachi and K\={o}shima, Nadia and Eden, Talya and Ben-Eliezer, Omri and Seshadhri, C.},
title = {A Sublinear Algorithm for Approximate Shortest Paths in Large Networks},
year = {2025},
isbn = {9798400713293},
publisher = {Association for Computing Machinery},
address = {New York, NY, USA},
url = {https://doi.org/10.1145/3701551.3703512},
doi = {10.1145/3701551.3703512},
booktitle = {Proceedings of the Eighteenth ACM International Conference on Web Search and Data Mining},
pages = {20–29},
numpages = {10},
location = {Hannover, Germany},
series = {WSDM '25}
}

@inproceedings{chen2020nearly,
  title={Nearly optimal edge estimation with independent set queries},
  author={Chen, Xi and Levi, Amit and Waingarten, Erik},
  booktitle={Proceedings of the Fourteenth Annual ACM-SIAM Symposium on Discrete Algorithms},
  pages={2916--2935},
  year={2020},
  organization={SIAM}
}

@inproceedings{rashtchian2021average,
  title={Average-case communication complexity of statistical problems},
  author={Rashtchian, Cyrus and Woodruff, David and Ye, Peng and Zhu, Hanlin},
  booktitle={Conference on Learning Theory},
  pages={3859--3886},
  year={2021},
  organization={PMLR}
}

@article{asai2023self,
  title={Self-rag: Learning to retrieve, generate, and critique through self-reflection},
  author={Asai, Akari and Wu, Zeqiu and Wang, Yizhong and Sil, Avirup and Hajishirzi, Hannaneh},
  journal={arXiv preprint arXiv:2310.11511},
  year={2023}
}

@article{su2024bright,
  title={Bright: A realistic and challenging benchmark for reasoning-intensive retrieval},
  author={Su, Hongjin and Yen, Howard and Xia, Mengzhou and Shi, Weijia and Muennighoff, Niklas and Wang, Han-yu and Liu, Haisu and Shi, Quan and Siegel, Zachary S and Tang, Michael and others},
  journal={arXiv preprint arXiv:2407.12883},
  year={2024}
}

@article{liu2025understanding,
  title={Understanding r1-zero-like training: A critical perspective},
  author={Liu, Zichen and Chen, Changyu and Li, Wenjun and Qi, Penghui and Pang, Tianyu and Du, Chao and Lee, Wee Sun and Lin, Min},
  journal={arXiv preprint arXiv:2503.20783},
  year={2025}
}

@article{ma2025benchmarking,
  title={Benchmarking Abstract and Reasoning Abilities Through A Theoretical Perspective},
  author={Ma, Qingchuan and Wu, Yuhang and Zheng, Xiawu and Ji, Rongrong},
  journal={arXiv preprint arXiv:2505.23833},
  year={2025}
}

@inproceedings{xie2023adaptive,
  title={Adaptive chameleon or stubborn sloth: Revealing the behavior of large language models in knowledge conflicts},
  author={Xie, Jian and Zhang, Kai and Chen, Jiangjie and Lou, Renze and Su, Yu},
  booktitle={International Conference on Learning Representations (ICLR)},
  year={2024}
}

@article{wei2024long,
  title={Long-form factuality in large language models},
  author={Wei, Jerry and Yang, Chengrun and Song, Xinying and Lu, Yifeng and Hu, Nathan and Huang, Jie and Tran, Dustin and Peng, Daiyi and Liu, Ruibo and Huang, Da and others},
  journal={Advances in Neural Information Processing Systems},
  volume={37},
  pages={80756--80827},
  year={2024}
}

@article{pan2024unifying,
  title={Unifying large language models and knowledge graphs: A roadmap},
  author={Pan, Shirui and Luo, Linhao and Wang, Yufei and Chen, Chen and Wang, Jiapu and Wu, Xindong},
  journal={IEEE Transactions on Knowledge and Data Engineering},
  volume={36},
  number={7},
  pages={3580--3599},
  year={2024},
  publisher={IEEE}
}

@inproceedings{guu2020retrieval,
  title={Retrieval augmented language model pre-training},
  author={Guu, Kelvin and Lee, Kenton and Tung, Zora and Pasupat, Panupong and Chang, Mingwei},
  booktitle={International conference on machine learning},
  pages={3929--3938},
  year={2020},
  organization={PMLR}
}

@article{zhou2025depth,
  title={In-depth Analysis of Graph-based RAG in a Unified Framework},
  author={Zhou, Yingli and Su, Yaodong and Sun, Youran and Wang, Shu and Wang, Taotao and He, Runyuan and Zhang, Yongwei and Liang, Sicong and Liu, Xilin and Ma, Yuchi and others},
  journal={arXiv preprint arXiv:2503.04338},
  year={2025}
}

@article{min2019knowledge,
  title={Knowledge guided text retrieval and reading for open domain question answering},
  author={Min, Sewon and Chen, Danqi and Zettlemoyer, Luke and Hajishirzi, Hannaneh},
  journal={arXiv preprint arXiv:1911.03868},
  year={2019}
}

@article{lewis2020retrieval,
  title={Retrieval-augmented generation for knowledge-intensive nlp tasks},
  author={Lewis, Patrick and Perez, Ethan and Piktus, Aleksandra and Petroni, Fabio and Karpukhin, Vladimir and Goyal, Naman and K{\"u}ttler, Heinrich and Lewis, Mike and Yih, Wen-tau and Rockt{\"a}schel, Tim and others},
  journal={Advances in Neural Information Processing Systems},
  volume={33},
  pages={9459--9474},
  year={2020}
}

@article{gutierrez2025rag,
  title={From rag to memory: Non-parametric continual learning for large language models},
  author={Guti{\'e}rrez, Bernal Jim{\'e}nez and Shu, Yiheng and Qi, Weijian and Zhou, Sizhe and Su, Yu},
  journal={arXiv preprint arXiv:2502.14802},
  year={2025}
}

@article{yang2024synthetic,
  title={Synthetic continued pretraining},
  author={Yang, Zitong and Band, Neil and Li, Shuangping and Candes, Emmanuel and Hashimoto, Tatsunori},
  journal={arXiv preprint arXiv:2409.07431},
  year={2024}
}

@article{joren2024sufficient,
  title={Sufficient Context: A New Lens on Retrieval Augmented Generation Systems},
  author={Joren, Hailey and Zhang, Jianyi and Ferng, Chun-Sung and Juan, Da-Cheng and Taly, Ankur and Rashtchian, Cyrus},
  journal={arXiv preprint arXiv:2411.06037},
  year={2024}
}

@misc{kim2025metastabledynamicschainofthoughtreasoning,
      title={Metastable Dynamics of Chain-of-Thought Reasoning: Provable Benefits of Search, RL and Distillation}, 
      author={Juno Kim and Denny Wu and Jason Lee and Taiji Suzuki},
      year={2025},
      eprint={2502.01694},
      archivePrefix={arXiv},
      primaryClass={cs.AI},
      url={https://arxiv.org/abs/2502.01694}, 
}

@article{alon2023sublinear,
  title={Sublinear time shortest path in expander graphs},
  author={Alon, Noga and Gr{\o}nlund, Allan and J{\o}rgensen, S{\o}ren Fuglede and Larsen, Kasper Green},
  journal={arXiv preprint arXiv:2307.06113},
  year={2023}
}

@misc{huang2023prodigyenablingincontextlearning,
      title={PRODIGY: Enabling In-context Learning Over Graphs}, 
      author={Qian Huang and Hongyu Ren and Peng Chen and Gregor Kržmanc and Daniel Zeng and Percy Liang and Jure Leskovec},
      year={2023},
      eprint={2305.12600},
      archivePrefix={arXiv},
      primaryClass={cs.LG},
      url={https://arxiv.org/abs/2305.12600}, 
}

@misc{wu2024avataroptimizingllmagents,
      title={AvaTaR: Optimizing LLM Agents for Tool Usage via Contrastive Reasoning}, 
      author={Shirley Wu and Shiyu Zhao and Qian Huang and Kexin Huang and Michihiro Yasunaga and Kaidi Cao and Vassilis N. Ioannidis and Karthik Subbian and Jure Leskovec and James Zou},
      year={2024},
      eprint={2406.11200},
      archivePrefix={arXiv},
      primaryClass={cs.LG},
      url={https://arxiv.org/abs/2406.11200}, 
}

@article{mirtaheri2025let,
  title={Let Me Think! A Long Chain-of-Thought Can Be Worth Exponentially Many Short Ones},
  author={Mirtaheri, Parsa and Edelman, Ezra and Jelassi, Samy and Malach, Eran and Boix-Adsera, Enric},
  journal={arXiv preprint arXiv:2505.21825},
  year={2025}
}

@misc{shalevshwartz2025reasoningsuperintelligencesearchtheoreticperspective,
      title={From Reasoning to Super-Intelligence: A Search-Theoretic Perspective}, 
      author={Shai Shalev-Shwartz and Amnon Shashua},
      year={2025},
      eprint={2507.15865},
      archivePrefix={arXiv},
      primaryClass={cs.AI},
      url={https://arxiv.org/abs/2507.15865}, 
}

@misc{setlur2024rewardingprogressscalingautomated,
      title={Rewarding Progress: Scaling Automated Process Verifiers for LLM Reasoning}, 
      author={Amrith Setlur and Chirag Nagpal and Adam Fisch and Xinyang Geng and Jacob Eisenstein and Rishabh Agarwal and Alekh Agarwal and Jonathan Berant and Aviral Kumar},
      year={2024},
      eprint={2410.08146},
      archivePrefix={arXiv},
      primaryClass={cs.LG},
      url={https://arxiv.org/abs/2410.08146}, 
}

@misc{koga2025privacypreservingretrievalaugmentedgenerationdifferential,
      title={Privacy-Preserving Retrieval-Augmented Generation with Differential Privacy}, 
      author={Tatsuki Koga and Ruihan Wu and Kamalika Chaudhuri},
      year={2025},
      eprint={2412.04697},
      archivePrefix={arXiv},
      primaryClass={cs.CR},
      url={https://arxiv.org/abs/2412.04697}, 
}

@inproceedings{balcan2017learning,
  title={Learning-theoretic foundations of algorithm configuration for combinatorial partitioning problems},
  author={Balcan, Maria-Florina and Nagarajan, Vaishnavh and Vitercik, Ellen and White, Colin},
  booktitle={Conference on Learning Theory},
  pages={213--274},
  year={2017},
  organization={PMLR}
}

@article{balcan2020data,
  title={Data-driven algorithm design},
  author={Balcan, Maria-Florina},
  journal={arXiv preprint arXiv:2011.07177},
  year={2020}
}

@article{balcan2024much,
  title={How much data is sufficient to learn high-performing algorithms?},
  author={Balcan, Maria-Florina and Deblasio, Dan and Dick, Travis and Kingsford, Carl and Sandholm, Tuomas and Vitercik, Ellen},
  journal={Journal of the ACM},
  volume={71},
  number={5},
  pages={1--58},
  year={2024},
  publisher={ACM New York, NY}
}

@inproceedings{balcan2025learning,
  title={On Learning Verifiers and Implications to Chain-of-Thought
Reasoning},
  author={Balcan, Maria-Florina and Blum, Avrim and Li, Zhiyuan and Sharma, Dravyansh},
  booktitle={NeurIPS},
  year={2025}
}

@misc{uesato2022solvingmathwordproblems,
      title={Solving math word problems with process- and outcome-based feedback}, 
      author={Jonathan Uesato and Nate Kushman and Ramana Kumar and Francis Song and Noah Siegel and Lisa Wang and Antonia Creswell and Geoffrey Irving and Irina Higgins},
      year={2022},
      eprint={2211.14275},
      archivePrefix={arXiv},
      primaryClass={cs.LG},
      url={https://arxiv.org/abs/2211.14275}, 
}

@misc{lightman2023letsverifystepstep,
      title={Let's Verify Step by Step}, 
      author={Hunter Lightman and Vineet Kosaraju and Yura Burda and Harri Edwards and Bowen Baker and Teddy Lee and Jan Leike and John Schulman and Ilya Sutskever and Karl Cobbe},
      year={2023},
      eprint={2305.20050},
      archivePrefix={arXiv},
      primaryClass={cs.LG},
      url={https://arxiv.org/abs/2305.20050}, 
}

@article{10.1023/A:1022821128753,
author = {Angluin, Dana},
title = {Queries and Concept Learning},
year = {1988},
issue_date = {April 1988},
publisher = {Kluwer Academic Publishers},
address = {USA},
volume = {2},
number = {4},
issn = {0885-6125},
url = {https://doi.org/10.1023/A:1022821128753},
doi = {10.1023/A:1022821128753},
journal = {Mach. Learn.},
month = apr,
pages = {319–342},
numpages = {24}
}

@misc{DynamicRAG2025,
      title={Dynamic and Parametric Retrieval-Augmented Generation}, 
      author={Weihang Su and Qingyao Ai and Jingtao Zhan and Qian Dong and Yiqun Liu},
      year={2025},
      eprint={2506.06704},
      archivePrefix={arXiv},
      primaryClass={cs.CL},
      url={https://arxiv.org/abs/2506.06704}, 
}

@misc{CRAG2024,
      title={Corrective Retrieval Augmented Generation}, 
      author={Shi-Qi Yan and Jia-Chen Gu and Yun Zhu and Zhen-Hua Ling},
      year={2024},
      eprint={2401.15884},
      archivePrefix={arXiv},
      primaryClass={cs.CL},
      url={https://arxiv.org/abs/2401.15884}, 
}

@misc{HRL22,
      title={Few-shot Relational Reasoning via Connection Subgraph Pretraining}, 
      author={Qian Huang and Hongyu Ren and Jure Leskovec},
      year={2022},
      eprint={2210.06722},
      archivePrefix={arXiv},
      primaryClass={cs.LG},
      url={https://arxiv.org/abs/2210.06722}, 
}

@article{comanici2025gemini,
  title={Gemini 2.5: Pushing the frontier with advanced reasoning, multimodality, long context, and next generation agentic capabilities},
  author={Comanici, Gheorghe and Bieber, Eric and Schaekermann, Mike and Pasupat, Ice and Sachdeva, Noveen and Dhillon, Inderjit and Blistein, Marcel and Ram, Ori and Zhang, Dan and Rosen, Evan and others},
  journal={arXiv preprint arXiv:2507.06261},
  year={2025}
}

@article{yang2025qwen3,
  title={Qwen3 technical report},
  author={Yang, An and Li, Anfeng and Yang, Baosong and Zhang, Beichen and Hui, Binyuan and Zheng, Bo and Yu, Bowen and Gao, Chang and Huang, Chengen and Lv, Chenxu and others},
  journal={arXiv preprint arXiv:2505.09388},
  year={2025}
}

@misc{deepseekai2025deepseekr1incentivizingreasoningcapability,
      title={DeepSeek-R1: Incentivizing Reasoning Capability in LLMs via Reinforcement Learning}, 
      author={DeepSeek-AI and Daya Guo and Dejian Yang and Haowei Zhang and Junxiao Song and Ruoyu Zhang and Runxin Xu and Qihao Zhu and Shirong Ma and Peiyi Wang and Xiao Bi and Xiaokang Zhang and Xingkai Yu and Yu Wu and Z. F. Wu and Zhibin Gou and Zhihong Shao and Zhuoshu Li and Ziyi Gao and Aixin Liu and Bing Xue and Bingxuan Wang and Bochao Wu and Bei Feng and Chengda Lu and Chenggang Zhao and Chengqi Deng and Chenyu Zhang and Chong Ruan and Damai Dai and Deli Chen and Dongjie Ji and Erhang Li and Fangyun Lin and Fucong Dai and Fuli Luo and Guangbo Hao and Guanting Chen and Guowei Li and H. Zhang and Han Bao and Hanwei Xu and Haocheng Wang and Honghui Ding and Huajian Xin and Huazuo Gao and Hui Qu and Hui Li and Jianzhong Guo and Jiashi Li and Jiawei Wang and Jingchang Chen and Jingyang Yuan and Junjie Qiu and Junlong Li and J. L. Cai and Jiaqi Ni and Jian Liang and Jin Chen and Kai Dong and Kai Hu and Kaige Gao and Kang Guan and Kexin Huang and Kuai Yu and Lean Wang and Lecong Zhang and Liang Zhao and Litong Wang and Liyue Zhang and Lei Xu and Leyi Xia and Mingchuan Zhang and Minghua Zhang and Minghui Tang and Meng Li and Miaojun Wang and Mingming Li and Ning Tian and Panpan Huang and Peng Zhang and Qiancheng Wang and Qinyu Chen and Qiushi Du and Ruiqi Ge and Ruisong Zhang and Ruizhe Pan and Runji Wang and R. J. Chen and R. L. Jin and Ruyi Chen and Shanghao Lu and Shangyan Zhou and Shanhuang Chen and Shengfeng Ye and Shiyu Wang and Shuiping Yu and Shunfeng Zhou and Shuting Pan and S. S. Li and Shuang Zhou and Shaoqing Wu and Shengfeng Ye and Tao Yun and Tian Pei and Tianyu Sun and T. Wang and Wangding Zeng and Wanjia Zhao and Wen Liu and Wenfeng Liang and Wenjun Gao and Wenqin Yu and Wentao Zhang and W. L. Xiao and Wei An and Xiaodong Liu and Xiaohan Wang and Xiaokang Chen and Xiaotao Nie and Xin Cheng and Xin Liu and Xin Xie and Xingchao Liu and Xinyu Yang and Xinyuan Li and Xuecheng Su and Xuheng Lin and X. Q. Li and Xiangyue Jin and Xiaojin Shen and Xiaosha Chen and Xiaowen Sun and Xiaoxiang Wang and Xinnan Song and Xinyi Zhou and Xianzu Wang and Xinxia Shan and Y. K. Li and Y. Q. Wang and Y. X. Wei and Yang Zhang and Yanhong Xu and Yao Li and Yao Zhao and Yaofeng Sun and Yaohui Wang and Yi Yu and Yichao Zhang and Yifan Shi and Yiliang Xiong and Ying He and Yishi Piao and Yisong Wang and Yixuan Tan and Yiyang Ma and Yiyuan Liu and Yongqiang Guo and Yuan Ou and Yuduan Wang and Yue Gong and Yuheng Zou and Yujia He and Yunfan Xiong and Yuxiang Luo and Yuxiang You and Yuxuan Liu and Yuyang Zhou and Y. X. Zhu and Yanhong Xu and Yanping Huang and Yaohui Li and Yi Zheng and Yuchen Zhu and Yunxian Ma and Ying Tang and Yukun Zha and Yuting Yan and Z. Z. Ren and Zehui Ren and Zhangli Sha and Zhe Fu and Zhean Xu and Zhenda Xie and Zhengyan Zhang and Zhewen Hao and Zhicheng Ma and Zhigang Yan and Zhiyu Wu and Zihui Gu and Zijia Zhu and Zijun Liu and Zilin Li and Ziwei Xie and Ziyang Song and Zizheng Pan and Zhen Huang and Zhipeng Xu and Zhongyu Zhang and Zhen Zhang},
      year={2025},
      eprint={2501.12948},
      archivePrefix={arXiv},
      primaryClass={cs.CL},
      url={https://arxiv.org/abs/2501.12948}, 
}

@book{Frieze_Karoński_2015, place={Cambridge}, title={Introduction to Random Graphs}, publisher={Cambridge University Press}, author={Frieze, Alan and Karoński, Michał}, year={2015}}

@article{Vu2007SpectralNorm,
  author  = {Vu, Van H.},
  title   = {Spectral norm of random matrices},
  journal = {Combinatorica},
  year    = {2007},
  volume  = {27},
  number  = {6},
  pages   = {721--736},
  month   = nov,
  doi     = {10.1007/s00493-007-2190-z},
  url     = {https://doi.org/10.1007/s00493-007-2190-z},
  issn    = {1439-6912},
}

@misc{rohatgi2025tamingimperfectprocessverifiers,
      title={Taming Imperfect Process Verifiers: A Sampling Perspective on Backtracking}, 
      author={Dhruv Rohatgi and Abhishek Shetty and Donya Saless and Yuchen Li and Ankur Moitra and Andrej Risteski and Dylan J. Foster},
      year={2025},
      eprint={2510.03149},
      archivePrefix={arXiv},
      primaryClass={cs.LG},
      url={https://arxiv.org/abs/2510.03149}, 
}

@misc{yao2023treethoughtsdeliberateproblem,
      title={Tree of Thoughts: Deliberate Problem Solving with Large Language Models}, 
      author={Shunyu Yao and Dian Yu and Jeffrey Zhao and Izhak Shafran and Thomas L. Griffiths and Yuan Cao and Karthik Narasimhan},
      year={2023},
      eprint={2305.10601},
      archivePrefix={arXiv},
      primaryClass={cs.CL},
      url={https://arxiv.org/abs/2305.10601}, 
}
\newpage
\appendix
\section{Appendix}\label{Appendix:proofs}

\begin{definition}[Prior-Aware Retrieval Oracle with Memory]
\label{def:RetrievalOracleKnowledge}
    Let $G^*=(V,E^*)$ be the ground-truth graph and $G=(V,E)$ a pretrained subgraph.  
    The \emph{retrieval oracle} is specified by the family $\{\pi_u^{G^*}\}_{u\in V}$ where each $\pi_u^{G^*}$ is the uniform probability distribution over neighbors of $u$ in $G^*$.

    The oracle never repeats an edge it has already revealed or one present in the prior graph. It maintains a set of seen edges \(E_{\mathrm{seen}}\subseteq E^*\), initialized as \(E_{\mathrm{seen}} := E\).  
    For each vertex \(u\in V\), let  
    \[N_{G^*}^{\mathrm{unseen}}(u)=\{v\in N_{G^*}(u):(u,v)\notin E_{\mathrm{seen}}\},\]  
    and let $\pi_u^{G^*,\mathrm{unseen}}$ denote the restriction of $\pi_u^{G^*}$ to this set (renormalized).  
    On query $u\in V$, the oracle returns
    \[
    \mathcal{O}_{G^*}(u)=
    \begin{cases}
        (u,v) & \text{with probability }\pi_u^{G^*,\mathrm{unseen}}(v),\\[3pt]
        \bot & \text{if } N_{G^*}^{\mathrm{unseen}}(u)=\emptyset.
    \end{cases}
    \]
    After returning $(u,v)$, the oracle updates \(E_{\mathrm{seen}}\) and its dependent sets accordingly.
\end{definition}

\subsection{Proof of \Cref{proposition:RandomizedAdaptive}}
\label{appendix:RandomizedAdaptive}
\begin{proof}
    By Yao's Minimax Principle, it suffices to prove a lower bound for the best deterministic algorithm against an input distribution that is hard on average.
    In the distribution from the proposition the bridge endpoints $(u,v)$ are chosen uniformly and independently.
    A pair of vertices $(s,t)$ chosen uniformly at random falls into one of following cases. Either both $s$ and $t$ lie in the same partition which occurs with probability approaching $1/2$ and a path already exists within the prior and no retrieval queries is needed. 
    Or, $s$ and $t$ lie in different partitions, and any path from $s$ to $t$ must traverse the hidden bridge $(u,v)$.
    Thus, to achieve an overall success probability of at least $2/3$ on a uniformly random pair, an algorithm must have at least a constant success probability on the inter-star instances. We establish a lower bound for this sub-problem which reduces to identifying the bridge. For the sake of proving a lower bound suppose access to a the prior aware retrieval oracle with memory (\Cref{def:RetrievalOracleKnowledge}). 
    
    We first demonstrate that an optimal algorithm will focus on identifying one of the bridge's endpoints. That is, the learner queries the leaves of one star sequentially in order to find the target hidden leaf $u \in S\setminus\{c_s\}$ among $\frac{n}{2} - 1$ such candidates. 
    To see why this strategy is optimal, suppose the algorithm instead made $q_1$ queries to the leaves of the star centered at $c_s$ and $q-q_1$ queries to the leaves of the star centered at $c_t$. The probability of missing the target leaf on the star centered at $c_s$ is $1-\frac{q_1}{n/2-1}$ and the probability of missing the special leaf on the star centered at $c_t$ is $1-\frac{q-q_1}{n/2-1}$. Hence the success probability is $$1- (1-\frac{q-q_1}{n/2-1})\cdot(1-\frac{q_1}{n/2-1}) = \frac{q}{\,n/2-1\,}-\frac{q_1(q-q_1)}{(n/2-1)^2}.$$
    
    Now, consider a learner that focuses on the leaves of one star instead and uses $q$ retrieval queries.  Since the retriever is prior aware with memory, this process is equivalent to choosing $q$ leaves all at once from a set of $n/2 - 1$ leaves to find the target leaf.
    \[
    \Pr[\text{bridge discovered within } q \text{ queries}] = \frac{q}{n/2 - 1},
    \]
    Demonstrating the latter strategy is optimal since $\frac{q}{\,n/2-1\,}-\frac{q_1(q-q_1)}{(n/2-1)^2}$ is strictly smaller than $\frac{q}{\,n/2-1\,}$.

    Therefore, for an algorithm to succeed on the inter star instances with a constant probability, it must make $q = \Omega(n)$ queries (it is easy to see that an adaptive learner has the same query complexity as well).
    The overall expected query complexity is the average over both cases; thus, any algorithm that succeeds with an overall probability of at least $2/3$ must perform $\Omega(n)$ queries. 
\end{proof}

\subsection{Proof of \Cref{Proposition:birthdayparadox}}
\begin{proof}
    \label{birthdayparadox-lowererdos}
    Let $k_s,k_t$ be the number of queries at $s$ and $t$, and $Q=q_s+q_t$. The probability that the direct edge $(s,t)$ is revealed is at most $\frac{q_s}{n-1}+\frac{q_t}{n-1}\le \frac{Q}{n-1}$. To succeed with probability at least $1/2$ we need $Q=\Omega(n)$; thus, the learner needs to target finding a path of length at least two.
    
    Next, the proof establishes the lower bound by first considering the specific problem of finding a length-two path, whose structure motivates a more general argument.
    A query to a vertex $v \notin \{s,t\}$ finds a path only if the returned neighbor is $s$ (and $v$ is known to be a neighbor of $t$) or $t$ (and $v$ is known to be a neighbor of $s$); therefore, without loss of generality assume that $\mathcal{A}$ queries $s$ and $t$, and $Q=q_s+q_t$. After $q_s$ and $q_t$ queries, the discovered neighbor sets
    $S_Q, T_Q\subseteq V\setminus\{s,t\}$ satisfy $|S_Q|\le q_s$, $|T_Q| \le q_t$. 
    A length-2 path is found iff $S\cap T\neq\emptyset$, that is we need $\Pr(\text{success})=\Pr(|S_Q\cap T_Q|\ge1) \leq \mathbb{E}[|S\cap T|]
    \le  \frac{(n-2)q_sq_t}{(n-1)^2}$ to be at least $1/2$.
    For fixed total $Q=q_s+q_t$, the product $q_sq_t$ is maximized at $q_s=q_t=Q/2$; thus, for a success probability greater than half we need $ Q = \Omega(\sqrt{n})$. We note that this search for a common element between two incrementally revealed sets has the structure of the birthday paradox; we now generalize this to paths of longer lengths.
    The argument rests on reducing the path-finding problem to that of inducing a collision, defined as any event where a query returns a vertex that has already been discovered. Label $V \setminus \{s,t\}=\{1,\dots,n-2\}$, and in each round $i$ the algorithm chooses a vertex $u_i\in V$ and the oracle returns a random neighbor $v_i\in N_{G^*}(u_i)$. 
    At each round, place a ball in the bin of the queried vertex and in the bin of the returned neighbor. Let a \emph{connection collision} be the event that the \emph{returned neighbor} $v_i$ falls into a bin that already contains a ball from some earlier round. Note that finding a connection collision is a lower bound on finding a connected path, and by birthday paradox, any algorithm needs $\Omega(\sqrt{n})$ rounds in expectation before the first connection collision, and hence $\Omega(\sqrt{n})$ retrieval queries in expectation before it can find an $s$-$t$ path.  
    Therefore, the minimum expected number of queries required by is $\Omega(\sqrt{n})$, completing the proof.
\end{proof}

\begin{lemma} [$K$-Birthday Paradox]
\label{lem:k-birthdayparadox} 
    Throw $m$ balls independently and uniformly into $n$ bins, and let $C$ be the number of bins with at least $2$ balls. Then, $m =o(\sqrt{{Kn}})$ implies $C< K$ with high probability. 
\end{lemma}

\begin{proof}
    Let $X_i$ be the number of balls in bin $i$. The event of interest is $  \sum_{i=1}^n \mathbf{1}\{X_i\ge 2\} \ge K$ which is monotone in $m$. By standard Poissonization,
    \[
    \Pr\Big(\sum_{i=1}^n \mathbf{1}\{X_i\ge 2\}\ge K\Big)
     \le 2 \Pr\Big(\sum_{i=1}^n \mathbf{1}\{Y_i\ge 2\}\ge K\Big),
    \]
    where $Y_1,\dots,Y_n$ are i.i.d.\ $\mathrm{Poi}(\lambda)$ with $\lambda=m/n$. Let $Z_i=\mathbf{1}\{Y_i\ge 2\}$ and $\mu=\mathbb{E}[Z_i]=\Pr(Y_i\ge 2)$. Observe that
    \[
    \mu=1-e^{-\lambda}(1+\lambda) \le 1-(1-\lambda)(1+\lambda) = \lambda^2 = (\frac{m}{n})^2
    \]
    and we have $\mathbb{E}[ \sum_{i=1}^n Z_i ] =  n\mu \le \frac{m^2}{n}$. Therefore, for any $\epsilon>0$,  if $m \le \sqrt{\frac{K n}{1+\epsilon}}$ then $K \geq (1+\epsilon)\, \mathbb{E}[ \sum_{i=1}^n Z_i ]$. 
    Thus, by Chernoff bound 
    \[
    \Pr\Big(\sum_{i=1}^n Z_i \ge K\Big) \le  \Pr\Big( \sum_{i=1}^n Z_i \geq (1+\gamma)n\mu\Big) \leq \exp  (- n\mu\epsilon^2/3 ),
    \]
    completing the proof.
\end{proof}

\begin{remark}
\label{rem:k-edge-disjoint}
In the setting of \Cref{Proposition:birthdayparadox}, any algorithm that finds $K$ edge disjoint $s$–$t$ paths with constant probability requires $\Omega(\sqrt{K n})$ retrieval queries.
First, note that there can be \emph{at most one} length one path; therefore, for $K\ge 2$ at least $K-1$ of the paths must have length more than one. Moreover, finding the direct edge path requires $\Omega(n)$ queries as stated before.
For length two paths, let $q_s,q_t$ be the numbers of queries at $s$ and $t$, and set $Q:=q_s+q_t$. As in the proof above,
\[
\mathbb{E}\big[\,|S_Q\cap T_Q|\,\big]
=\sum_{v}\Pr(v\in S_Q\cap T_Q)
\le \frac{(n-2)\,q_s q_t}{(n-1)^2}
\le \frac{(n-2)\,Q^2}{4(n-1)^2}.
\]
By Markov’s inequality,
\[
\Pr\big(|S_Q\cap T_Q|\ge K\big)\ \le\ \frac{\mathbb{E}[|S_Q\cap T_Q|]}{K}
\ \le\ \frac{(n-2)Q^2}{4K(n-1)^2}.
\]
Thus achieving constant success probability for $K$ edge disjoint $s$-$t$ paths requires
$Q=\Omega(\sqrt{K n})$.

For longer paths, we can again use the balls and bins argument and reduce it to $K$-birthday paradox (see \Cref{lem:k-birthdayparadox}). Each edge disjoint path needs at least one distinct intermediate vertex, so we need at least $K$ connection collisions, that is, at least $K$ bins with at least two balls which requires $Q=\Omega(\sqrt{K n})$ as shown in \cref{lem:k-birthdayparadox}
\end{remark}

\subsection{Proof of \Cref{theorem:supernodelowerbound}}
\label{erdosquery}   
\begin{proof}
    Our proof adapts Theorem 4 of \cite{alon2023sublinear}. 
    Given each edge from $G^*$ is retained in the pretrained graph $G$ with probability $\eta$, the pretrained graph $G$ itself is an Erd\H{o}s–R\'enyi random graph $G \sim \ER(n, p_{\text{partial}})$ with $p_{\text{partial}}  = p \cdot \eta < \frac{1}{n}$.
    It is known that in this subcritical regime an Erd\H{o}s–R\'enyi random graph  $G \sim G(n,p)$ with $p < 1 /n$ , the largest connected component has size $O(\log n)$ with high probability.
    
    We note that the generation process is equivalent to first sampling $G\sim\ER(n,p\eta)$ and then, to form $G^*$, adding each edge not present in $G$ independently with probability $q=\frac{p-p\eta}{1-p\eta}$. This ensures $G^*$ is a valid $\ER(n,p)$ graph. This allows us to first condition on the realization of $G$ (and thus the partition of the vertices and connected components) and then analyze the properties of $G^*$ and the meta graph we introduce in what follows.
     
    For the lower bound, take the extremal case where all components have size $C\log n$ for a fixed absolute constant \(C\ge 1\). Contracting components yields $n'=\frac{n}{C\log n}$ \emph{super-nodes}.
    Let $G'$ be the meta-graph on the super nodes. Two super nodes are adjacent in $G'$ if there exists at least one cross edge in $G^*$ between their underlying components. By a union bound over at most $(C\log n)^2$ potential cross-edges between two components,
    \[
        \Pr[\text{edge in } G'] \le (C\log n)^2 p.
    \]
    With $p=\frac{\log n}{n}$ and $n'=\frac{n}{C\log n}$, this simplifies to
    \[
        \Pr[\text{edge in } G'] \le \frac{C^2\,\log^3 n}{n} = \frac{C\,\log^2 n}{n'} \le  \frac{4C\,\log^2 n'}{n'} .
    \]
    Therefore $G'$ is no denser than $\ER \left(n',\, \frac{4C\,\log^2 n'}{n'}\right)$ for large $n'$.

    Let \( \mathcal{A}^* \) be a (possibly randomized) algorithm for computing an \(s\)-\(t\) path in \(G'\). Without loss of generality, we label vertices so that \(s=1\) and \(t=n'\) as it does not change the success probability.
    Let \( \alpha^* \) denote the probability that \( \mathcal{A}^* \) outputs a valid \( s \)-\( t \), that is, the path also exists in $G^*$.  
    Suppose \( \mathcal{A}^* \) has access to a node incident retrieval oracle. This oracle for a queried vertex $u$ returns the entire set of edges incident to $u$ in $G^*$ or $\bot$ if $u$ is isolated. Observe that node incident retrieval queries in $G'$ are equivalent to connected component incident retrieval (\Cref{def:ccincidentquery}) queries in $G'$. Let \( q \) be the worst case number of node incident retrieval queries made by \( \mathcal{A}^* \).
    
    Note that for \( \mathcal{A}^* \) making an expected \( q \) queries, one can make it worst case \( O(q) \) queries by decreasing \( \alpha^* \) by a small additive constant. Here the probability is over both the random choices of algorithm \( \mathcal{A}^* \) and the randomness of graph \( G' \). By linearity of expectation, we may fix the random choices of \( \mathcal{A}^* \) to obtain a deterministic algorithm \( \mathcal{A} \) that outputs a valid \( s \)-\( t \) path with probability \( \alpha \geq \alpha^* \). It thus suffices to prove an upper bound on \( \alpha \) for such deterministic \( \mathcal{A} \).
    
    For the graph \( G' \), let \( \pi(\mathcal{A}, G') \) denote the \emph{trace} of running the deterministic \( \mathcal{A} \) on \( G' \). If \( i_1(G'), \ldots, i_q(G') \) denotes the sequence of edges queried by \( \mathcal{A} \) on \( G \), and \( \mathcal{N}_1(G'), \ldots, \mathcal{N}_q(G') \) denotes the returned sets of edges, then
    \[
        \pi(\mathcal{A}, G') = \langle i_1(G'), \mathcal{N}_1(G'), i_2(G'), \ldots, i_q(G'), \mathcal{N}_q(G') \rangle.
    \]
    If we condition on a particular trace \( \tau = (i_1, N_1, i_2, \ldots, i_q, N_q) \), the distribution of \( G' \) conditioned on \( \pi(\mathcal{A}, G') = \tau \) is the same as if we condition on the set of edges incident to \( i_1, \ldots, i_q \) being \( N_1, \ldots, N_q \). This is because the algorithm \( \mathcal{A} \) is deterministic and the execution of \( \mathcal{A} \) is the same for all graphs \( G' \) with the same such sets of edges incident to \( i_1, \ldots, i_q \). Furthermore, no graph \( G' \) with a different set of incident edges for \( i_1, \ldots, i_q \) will result in the trace \( \tau \).

    For a trace \( \tau = (i_1, N_1, \ldots, i_q, N_q) \), call the trace \emph{connected} if there is a path from \( s \) to \( t \) using the discovered edges
    \[
        \bigcup_{j=1}^{q} N_j.
    \]
    Otherwise, call it \emph{disconnected}. Intuitively, if a trace is disconnected, then it is unlikely that \( \mathcal{A} \) will succeed in outputting a valid path connecting \( s \) and \( t \), as it has to guess some of the edges along such a path. Furthermore, if \( \mathcal{A} \) makes too few queries, then it is unlikely that the trace is connected. Letting \( \mathcal{A}(G') \) denote the output of \( \mathcal{A} \) on the graph \( G' \), we have for a random graph \( G' \) that
    
    \[
        \alpha = \Pr[ \mathcal{A}(G') \text{ is valid} ] \leq \Pr[ \pi(\mathcal{A}, G') \text{ is connected} ] + \Pr[ \mathcal{A}(G') \text{ is valid} \mid \pi(\mathcal{A}, G') \text{ is disconnected} ].
    \]
    We first bound $ \Pr[ \mathcal{A}(G') \text{ is valid} \mid \pi(\mathcal{A}, G') \text{ is disconnected} ].$
    For this, let \( \tau = (i_1, N_1, \ldots, i_q, N_q) \) be an arbitrary disconnected trace in the support of \( \pi(\mathcal{A}, G') \) when \( G' \) is an Erdős–Rényi random graph, where each edge is present with probability \( p' \geq \frac{4 C \log^2 n'}{n'}  \). Observe that the output of \( \mathcal{A} \) is determined from \( \tau \). Since \( \tau \) is disconnected, the path reported by \( \mathcal{A} \) on \( \tau \) must contain at least one edge \( (u,v) \) where neither \( u \) nor \( v \) is among \( \cup_j \{i_j\} \), or otherwise the output path is valid with probability \( 0 \) conditioned on \( \tau \). But conditioned on the trace \( \tau \), every edge that is not connected to \( \{i_1, \ldots, i_q\} \) is present independently with probability \( p' \). We thus conclude:
    \[
        \Pr[ \mathcal{A}(G') \text{ is valid} \mid \pi(\mathcal{A}, G') = \tau ] \leq p'.
    \]
    Since this holds for every disconnected \( \tau \), we conclude:
    \[
        \Pr[ \mathcal{A}(G') \text{ is valid} \mid \pi(\mathcal{A}, G') \text{ is disconnected} ] \leq p'.
    \]
    Next we bound the probability that \( \pi(\mathcal{A}, G') \) is connected. For this, define for \( 1 \leq k \leq q \):
    \[
         \pi_k(G') = \langle i_1(G'), \mathcal{N}_1(G'), i_2(G'), \ldots, i_k(G'), \mathcal{N}_k(G') \rangle.
    \]
    as the trace of \( \mathcal{A} \) on \( G' \) after the first \( k \) queries. As for \( \pi_k(G') \), we say that \( \pi_k(G') \) is connected if there is a path from \( s \) to \( t \) using the discovered edges
    \[
     E(\pi_k(G')) = \bigcup_{j=1}^k \mathcal{N}_j(G')
    \]
    and that it is disconnected otherwise. We further say that \( \pi_k(G') \) is \emph{useless} if it is both disconnected and \( |E(\pi_k(G'))| \leq 2p'n'k \). Since
    \[
        \Pr[ \pi_k(G') \text{ is disconnected} ] \geq \Pr[ \pi_k(G') \text{ is useless} ],
    \]
    we prove that \( \Pr[ \pi_k(G') \text{ is useless} ] \) is large. Therefore, we lower bound
    \[
        \Pr[ \pi_k(G') \text{ is useless} \mid \pi_{k-1}(G') \text{ is useless} ].
    \]
    Note that the base case \( \pi_0(G') \) is defined to be useless as \( s \) and \( t \) are not connected when no queries have been asked and also \( |E(\pi_0(G'))| = 0 \leq 2p'n' \cdot 0 = 0 \). Let \( \tau_{k-1} = (i_1, N_1, \ldots, i_{k-1}, N_{k-1}) \) be any useless trace. The query \( i_k = i_k(G') \) is uniquely determined when conditioning on \( \pi_{k-1}(G') = \tau_{k-1} \), and so is the edge set \( E_{k-1} = E(\pi_{k-1}(G')) \). Furthermore, we know that \( |E_{k-1}| \leq 2p'n'(k-1) \). We now bound the probability that the query discovers more than \( 2p'n' \) new edges. If \( i_k \) has already been queried, no new edges are discovered and the probability is \( 0 \). So assume \( i_k \notin \{ i_1, \ldots, i_{k-1} \} \). Now observe that conditioned on \( \pi_{k-1}(G') = \tau_{k-1} \), the edges \( (i_k, i) \) where \( i \notin \{ i_1, \ldots, i_{k-1} \} \) are independently included in \( G' \) with probability \( p' \) each. The number of new edges discovered is thus a sum of \( m \leq n' \) independent Bernoulli's \( X_1, \ldots, X_m \) with success probability \( p' \). A Chernoff bound implies
    \[
        \Pr\left[ \sum_i X_i > 2n'p' \right] < (e/4)^{n'p'} < e^{-n'p'/3}.
    \]
    Since we assume \( p' \geq  \frac{4 C \log^2 n'}{n'} \), this is at most \(  n'^{-4 C \log n'/3}  \).
    We now bound the probability that the discovered edges \( \mathcal{N}_k(G') \) makes \( s \) and \( t \) connected in \( E(\pi_k(G')) \). For this, let \( V_s \) denote the nodes in the connected component of \( s \) in the subgraph induced by the edges \( E_{k-1} \). Define \( V_t \) similarly. We split the analysis into three cases. First, if \( i_k \in V_s \), then \( \mathcal{N}_k(G') \) connects \( s \) and \( t \) if and only if one of the edges \( \{ i_k, v \} \) with \( v \in V_t \) is in \( G' \). Conditioned on \( \pi_{k-1}(G') = \tau_{k-1} \), each such edge is in \( G' \) independently either with probability \( 0 \), or with probability \( p' \) (depending on whether one of the end points is in \( \{ i_1, \ldots, i_{k-1} \} \)). A union bound implies that \( s \) and \( t \) are connected in \( E(\pi_k(G')) \) with probability at most \( p' |V_t| \). A symmetric argument upper bounds the probability by \( p' |V_s| \) in case \( i_k \in V_t \). Finally, if \( i_k \) is in neither of \( V_s \) and \( V_t \), it must have an edge to both a node in \( V_s \) and in \( V_t \) to connect \( s \) and \( t \). By independence, this happens with probability at most \( p'^2 |V_s| |V_t| \). We thus conclude that
    \[
        \Pr[ \pi_k(G') \text{ is connected} \mid \pi_{k-1}(G') = \tau_{k-1} ] \leq p' \max\{ |V_s|, |V_t| \} \leq p'(|E_{k-1}| + 1) \leq 2p'n'k.
    \]
    By union bound
    \[
        \Pr[ \pi_k(G') \text{ is useless} \mid \pi_{k-1}(G') \text{ is useless} ] \geq 1 - 2p'^2n'k - \frac{1}{n'^{4 C \log n'/3}} .
    \]
    Thus
    \[
        \Pr[ \pi_q(G') \text{ is useless} ] = \prod_{k=1}^q \Pr[ \pi_k(G') \text{ is useless} \mid \pi_{k-1}(G') \text{ is useless} ] 
    \]
    \[
        \geq \prod_{k=1}^q \left( 1 - 2p'^2n'k - \frac{1}{n'^{4 C \log n'/3}} \right)
    \]
    \[
        \geq 1 - \sum_{k=1}^q \left( 2p'^2n'k + \frac{1}{n'^{4 C \log n'/3}} \right)
    \]
    \[
        \geq 1 - p'^2n'q(q+1) - \frac{q}{n'^{4 C \log n'/3}}.
    \]
    It follows
    \[
        \Pr[ \pi(G') \text{ is connected} ] = 1 - \Pr[ \pi(G') \text{ is disconnected} ] \leq 1 - \Pr[ \pi(G') \text{ is useless} ] \leq p'^2 n' (q+1)^2 + \frac{q}{n'^{4C \log n'/3}}.
    \]
    For $q = o\left( \frac{1}{p' \sqrt{n'}} \right)$ and $p' \geq \frac{ 4C\log^2 n' }{ n' }$ node-incident queries in the meta-graph $G'$, the success probability remains $o(1)$. Therefore, $q = \Omega \left( \frac{1}{p \cdot \log^2 n \cdot \sqrt{n}} \right)$ connected component incident retrieval (\Cref{def:ccincidentquery}) queries in given $G$ are necessary.  
\end{proof}

\subsection{Proof of \Cref{RMT}}
\label{Proof:RMT}
\begin{proof}
We consider an Erd\H{o}s-R\'{e}nyi random graph $G \sim G(n,p)$, where $p$ satisfies
\begin{equation}
  p(1-p) \geq C \cdot \frac{\log^4(n)}{n}
  .
  \label{eq:pcond}
\end{equation}
and, $C>0$ is taken to be a sufficiently large constant.

\paragraph{Vertex degrees of the random graph.}
Let $\deg(i)$ denote the degree of vertex $i$ in $G$.
By Bernstein's inequality and a union bound, we have the following with probability at least $1 - 2/n$:
\begin{equation}
  \del{ 1 - \delta_0 } pn \leq \deg(i) \leq \del{ 1 + \delta_0 } pn
  \quad \text{for all vertices $i$ in $G$}
  \label{eq:degbound}
\end{equation}
where
\begin{equation*}
  \delta_0 = \delta_0(p,n) := \frac1n + 2\sqrt{\frac{(1-p)\ln(n)}{pn}} + \frac{2\ln(n)}{3pn} .
\end{equation*}
The assumption in \Cref{eq:pcond} implies that $\delta_0 = O(1/\log^{3/2}(n))$.

\paragraph{Deviation of the random adjacency matrix from its expectation.}

Let the $n \times n$ random matrix $A$ denote the adjacency matrix of $G$, so
\begin{equation*}
  A_{i,j} =
  \begin{cases}
    1 & \text{if $i \neq j$ and $\set{i,j}$ is an edge in $G$} ; \\
    0 & \text{otherwise} .
  \end{cases}
\end{equation*}
Let $X = A - \E(A)$, so we have the following:
\begin{align*}
  \abs{X_{i,j}} & \leq 1
                && \text{for all $1 \leq i \leq j \leq n$} ; \\
  \mathbb{E}(X_{i,j}) & = 0 && \text{for all $1 \leq i \leq j \leq n$} ; \\
  \text{var}(X_{i,j}) & = p(1-p) && \text{for all $1 \leq i < j \leq n$} .
\end{align*}
Then, using the assumption in \Cref{eq:pcond} and the above properties of the random matrix $X$, Theorem~1.4 of \cite{Vu2007SpectralNorm} implies that, there is a constant $C'>0$ such that with probability at least $1 - o(1)$,
\begin{equation*}
    \norm{X}_2 \leq 2 \sqrt{p(1-p)n} + C' (p(1-p)n)^{1/4} \log(n) .
\end{equation*}
Here, the norm on $X$ is the spectral norm (i.e., largest singular value).
For any $\varepsilon>0$, there is a large enough $C$ in \Cref{eq:pcond} such that the above inequality implies
\begin{equation}
    \norm{A -  \mathbb{E}(A)}_2 \leq (2+\varepsilon) \sqrt{p(1-p)n} .
\label{eq:opbound}
\end{equation}

We henceforth condition on the event that both \Cref{eq:degbound} and \Cref{eq:opbound} hold.

\paragraph{Eigenvalues and eigenvectors of the expected adjacency matrix.}

For a symmetric matrix $M$, let $\lambda_k(M)$ denote its $k$-th largest eigenvalue.
The matrix $\E(A)$ can be written as
\begin{equation*}
  \E(A) = pn uu^\T - p I_n ,
\end{equation*}
where $u := n^{-1/2} 1_n$, $1_n := (1,1,\dotsc,1)$ is the all-$1$s vector in $\R^n$, and $I_n$ is the $n \times n$ identity matrix.
Therefore, the largest eigenvalue of $\E(A)$ is $\lambda_1(\E(A)) = p(n-1)$, and $u$ is a corresponding (unit length) eigenvector.
All other eigenvalues of $\E(A)$ are $\lambda_k(\E(A)) = -p$, for $k \neq 1$, and the corresponding eigenvectors $u_\perp$ satisfy $1_n^\T u_\perp = 0$.

\paragraph{Eigenvalues of the random adjacency matrix.}
By Weyl's inequality,
\begin{equation*}
  \abs{\lambda_k(A) - \lambda_k(\E(A))} \leq \norm{A - \E(A)}_2
\end{equation*}
for all $k$.
Therefore, using \Cref{eq:opbound}, we find that
\begin{equation}
  \del{ 1 - \delta_1 } pn \leq \lambda_1(A) \leq \del{ 1 + \delta_1 } pn
  \label{eq:firstbound}
\end{equation}
where
\begin{equation*}
  \delta_1 = \delta_1(p,n) := \frac1n + (2+\varepsilon)\sqrt{\frac{1-p}{pn}} .
\end{equation*}
Furthermore, $\lambda(A) := \max\set{\lambda_2(A), \abs{\lambda_n(A)}}$ satisfies
\begin{equation}
  \lambda(A) \leq (2+\varepsilon+\delta_2) \sqrt{p(1-p)n}
  \label{eq:secondbound}
\end{equation}
where
\begin{equation*}
  \delta_2 = \delta_2(p,n) := \sqrt{\frac{p}{(1-p)n}} .
\end{equation*}
    The assumption in \Cref{eq:pcond} implies that $\delta_1 = O(1/\log^2(n))$ and $\delta_2 = O(1/\log^2(n))$.

\paragraph{Leading eigenvector of the random adjacency matrix.}
Let $v_1$ be any unit length eigenvector corresponding to the largest eigenvalue $\lambda_1(A)$ of $A$.
Recall that $u = n^{-1/2} 1_n$ is a unit length eigenvector corresponding to the largest eigenvalue of $\E(A)$.
We show that $v_1$ (or $-v_1$) is close to $u$ in terms of both the Euclidean norm as well as the $l^\infty$ norm.

The closeness of $v_1$ to $u$ in Euclidean norm follows from the Davis-Kahan $\sin(\Theta)$ theorem, but here we give a direct argument.
We can write
\begin{equation*}
  u = c_1 v_1 + c_2 v_\perp
\end{equation*}
for some unit vector $v_\perp$ orthogonal to $v_1$, and some coefficients $c_1 = u^\T v_1$ and $c_2$ satisfying $c_1^2 + c_2^2 = 1$.
Then
\begin{align*}
  (A - \E(A)) v_1
  & = \del*{ A - pn uu^\T -p I_n } v_1 \\
  & = \lambda_1(A) v_1 - c_1 pn u - p v_1 \\
  & = \lambda_1(A) v_1 - c_1 pn \del*{ c_1 v + c_2 v_\perp } - p v_1 \\
  & = (\lambda_1(A) - p - c_1^2 pn) v_1 - c_1 c_2 pn v_\perp .
\end{align*}
Since $v_1$ and $v_\perp$ are orthogonal, the Pythagorean theorem implies
\begin{equation*}
  \norm{(A - \E(A)) v_1}_2 \geq \abs{ \lambda_1(A) - p - c_1^2 pn } .
\end{equation*}
On the other hand, by \Cref{eq:opbound}, we have
\begin{equation*}
  \norm{(A - \E(A)) v_1}_2 \leq (2+\varepsilon) \sqrt{p(1-p)n} .
\end{equation*}
Therefore,
\begin{equation*}
  \lambda_1(A) - p - c_1^2 pn \leq (2+\varepsilon)\sqrt{p(1-p)n} ,
\end{equation*}
which, together with \Cref{eq:firstbound} and \Cref{eq:pcond}, implies
\begin{equation*}
  (u^\T v_1)^2
  = c_1^2
  \geq \frac{\lambda_1(A)}{pn} - \frac1n - (2+\varepsilon) \sqrt{\frac{1-p}{pn}}
  = 1 - 2\delta_1 .
\end{equation*}
In particular, this implies $\norm{\sign(c_1) v_1 - u}_2 = \sqrt{2(1 - \abs{c_1})} \leq \sqrt{2(1 - \sqrt{1 - 2\delta_1})} = O(\sqrt{\delta_1})$.

Now we show closeness of $v_1$ to $u$ in $l^\infty$ norm.
For any non-negative integer $k$, define the vector
\begin{equation*}
  u^{(k)} = (u_1^{(k)},\dotsc,u_n^{(k)}) := \frac1{\lambda_1(A)^k} A^k u .
\end{equation*}
Also define
\begin{equation*}
  \delta_0' = \delta_0'(p,n) := \frac{1+\delta_0}{1-\delta_1} - 1 .
\end{equation*}
The assumption in \ref{eq:pcond} implies that $\delta_0' = O(\delta_0) = O(1/\log^{3/2}(n))$.
We show, by induction, that for all $k \leq 1 + 1/(2\delta_0')$,
\begin{equation*}
  u_i^{(k)} \in
  \intcc*{
    \frac{1-2k\delta_0'}{\sqrt{n}},
    \frac{1+2k\delta_0'}{\sqrt{n}}
  }
  \quad \text{for all vertices $i$ in $G$}
  .
\end{equation*}
The base case $k = 0$ clearly holds by definition of $u^{(0)} = u$.
So assume the claim holds for $k-1$.
Then, for any vertex $i$,
\begin{align*}
  u_i^{(k)}
  = \frac1{\lambda_1(A)} \sum_{j=1}^n A_{i,j} u_j^{(k-1)}
  & \leq \frac{\deg(i)}{\lambda_1(A)} \cdot \frac{1 + 2(k-1)\delta_0'}{\sqrt{n}}
  \quad \text{(inductive hypothesis)} \\
  & \leq \frac{1+\delta_0}{1-\delta_1} \cdot \frac{1 + 2(k-1)\delta_0'}{\sqrt{n}}
  \quad \text{(by \Cref{eq:degbound,eq:firstbound})} \\
  & = \frac{(1+\delta_0')(1 + 2(k-1)\delta_0')}{\sqrt{n}} \\
  & \leq \frac{1 + 2k\delta_0'}{\sqrt{n}}
  \quad \text{(by the upper-bound on $k$)} .
\end{align*}
Similarly,
\begin{align*}
  u_i^{(k)}
  & \geq \frac{\deg(i)}{\lambda_1(A)} \cdot \frac{1 - 2(k-1)\delta_0'}{\sqrt{n}}
  \quad \text{(inductive hypothesis)} \\
  & \geq \frac{1-\delta_0}{1+\delta_1} \cdot \frac{1 - 2(k-1)\delta_0'}{\sqrt{n}}
  \quad \text{(by \Cref{eq:degbound,eq:firstbound})} \\
  & \geq \frac{(1-\delta_0')(1 - 2(k-1)\delta_0')}{\sqrt{n}}
  \quad \text{(by comparison with $\delta_0'$)} \\
  & \geq \frac{1 - 2k\delta_0'}{\sqrt{n}} .
\end{align*}
Therefore the claim holds for all $k \leq 1 + 1/(2\delta_0')$.

Since $\delta_0' = O(1/\log^{3/2}(n))$, we can choose (with foresight)
\begin{equation*}
  k \asymp \frac{\log(n)}{\log(1/\rho(A))} ,
\end{equation*}
where
\begin{equation*}
  \rho(A) := \frac{\lambda(A)}{\lambda_1(A)} \leq \frac{2+\varepsilon+\delta_2}{1-\delta_1} \sqrt{\frac{1-p}{pn}} = O(\delta_1) .
\end{equation*}
This choice of $k$ satisfies $k \leq 1 + 1/(2\delta_0')$.
Observe that
\begin{align*}
  u^{(k)} = \frac1{\lambda_1(A)^k} A^k u
  & = \frac1{\lambda_1(A)^k} A^k (c_1 v_1 + c_2 v_\perp) \\
  & = c_1 v_1 + c_2 \frac1{\lambda_1(A)^k} A^k v_\perp .
\end{align*}
Therefore
\begin{align*}
  \norm*{ \frac1{c_1} u^{(k)} - v_1 }_\infty
  & = \abs*{\frac{c_2}{c_1\lambda_1(A)^k}} \norm*{A^k v_\perp}_\infty \\
  & \leq \abs*{\frac{c_2}{c_1\lambda_1(A)^k}} \norm*{A^k v_\perp}_2 \\
  & \leq \abs*{\frac{c_2}{c_1}} \del*{\frac{\lambda(A)}{\lambda_1(A)}}^k
  = \abs*{\frac{c_2}{c_1}} \rho(A)^k .
\end{align*}
We also have
\begin{equation*}
  \norm{u^{(k)} - u}_\infty \leq \frac{2k\delta_0'}{\sqrt{n}}
  \quad \text{and} \quad
  \norm*{ \frac1{\abs{c_1}} u - u }_\infty = \del*{ \frac1{\abs{c_1}} - 1 } \frac1{\sqrt{n}} .
\end{equation*}
By the triangle inequality,
\begin{equation*}
  \norm{\sign(c_1) v_1 - u}_\infty
  \leq \del*{ \frac1{\abs{c_1}} - 1 } \frac1{\sqrt{n}} + \frac{2k\delta_0'}{\abs{c_1}\sqrt{n}} + \abs*{\frac{c_2}{c_1}} \rho(A)^k .
\end{equation*}
Now we use the specific choice of $k$ to conclude
\begin{equation}
  \norm{\sign(c_1) v_1 - u}_\infty \leq \frac{\epsilon_0}{\sqrt{n}}
  \label{eq:linfbound}
\end{equation}
where
\begin{equation*}
  \epsilon_0 =
  O\del*{ \frac{\delta_0 \log(n)}{\log(1/\rho(A))} } .
\end{equation*}
The assumption in \Cref{eq:pcond} implies that $\epsilon_0 = O(1/(\sqrt{\log(n)}\log\log(n)))$.

We can now prove the main result, using mostly the same argument as in the proof of Lemma~1 from~\citet{alon2023sublinear}.
Without loss of generality, we assume $\sign(c_1) = 1$ (else we replace $v_1$ with $-v_1$).
Fix any vertex $i$ in $G$, any $\delta \in \intoo{0,1}$, and any distance $d$.
Let $Z$ denote the subset of vertices $j$ such that the $(i,j)$-th entry of $A^d$ is zero.
In other words, there are no length $d$ paths from $i$ to vertices in $Z$.
Let $1_Z$ be the $\set{0,1}$-characteristic vector for $Z$, i.e., the $j$-th component of $1_Z$ is $1$ if and only if $j \in Z$.
We can write
\begin{equation*}
  1_Z = b_1 v_1 + b_2 v_\perp
\end{equation*}
for some unit vector $v_\perp$ orthogonal to $v_1$, and some coefficients $b_1 = 1_Z^\T v_1$ and $b_2$ satisfying $b_1^2 + b_2^2 = \norm{1_Z}_2^2 = \card{Z}$.

Let $e_j$ denote the $j$-th coordinate basis vector.
Note that by \Cref{eq:linfbound}, we have
\begin{equation*}
  e_j^\T v_1 \geq \frac{1-\epsilon_0}{\sqrt{n}}
  \quad \text{and} \quad
  b_1 = 1_Z^\T v_1 = \sum_{j \in Z} e_j^\T v_1 \geq \frac{\card{Z}(1-\epsilon_0)}{\sqrt{n}}
  .
\end{equation*}
Therefore
\begin{align*}
  e_i^\T A^d 1_Z
  & = e_i^\T A^d (b_1 v_1 + b_2 v_\perp) \\
  & = b_1 \lambda_1(A)^d e_i^\T v_1 + b_2 e_i^\T A^d v_\perp \\
  & \geq \lambda_1(A)^d \card{Z} \frac{(1-\epsilon_0)^2}{n} - \abs{b_2} \lambda(A)^d \\
  & \geq \lambda_1(A)^d \card{Z} \frac{(1-\epsilon_0)^2}{n} - \sqrt{\card{Z}} \lambda(A)^d .
\end{align*}
On the other hand, we have $e_i^\T A^d 1_Z = 0$ since the $(i,j)$-th entry of $A^d$ is zero for all $j \in Z$.
Combining with the above inequality, we have
\begin{equation*}
  \lambda_1(A)^d \card{Z} \frac{(1-\epsilon_0)^2}{n} - \sqrt{\card{Z}} \lambda(A)^d \leq 0 ,
\end{equation*}
which rearranges to
\begin{equation*}
  \frac{\card{Z}}{n}
  \leq \frac{n}{(1-\epsilon_0)^4} \rho(A)^{2d}
  .
\end{equation*}
The right-hand side is at most $\delta$ provided that
\begin{equation*}
  d \geq \frac12 \cdot \frac{\log\del*{ \frac{n}{\delta(1-\epsilon_0)^4} }}{\log(1/\rho(A))} .
\end{equation*}
We conclude that there are at most $\delta n$ vertices with distance from $i$ more than
\begin{equation*}
  \frac12 \cdot \frac{\log\del*{ \frac{n}{\delta(1-\epsilon_0)^4(1-\rho(A)^2)} }}{\log(1/\rho(A))} .
\end{equation*}

This implies that for any vertex $i$, for at least $(1-\delta)n$ other vertices $j$, the number of nodes visited by the double BFS algorithm is at most
\begin{align*}
  \del*{ \max_i \deg(i) }^{\ceil*{\frac14 \cdot \frac{\log\del*{ \frac{n}{\delta(1-\epsilon_0)^4(1-\rho(A)^2)} }}{\log(1/\rho(A))}}}
  & \leq \del*{ (1+\delta_0)pn }^{\ceil*{\frac14 \cdot \frac{\log\del*{ \frac{n}{\delta(1-\epsilon_0)^4(1-O(\delta_1^2))} }}{\log(1/\rho(A))}}} \\
  & = O\del*{ n/\delta }^{\frac14 \cdot \frac{\log((1+\delta_0)pn)}{\log(1/\rho(A))}} .
\end{align*}
We can simplify the exponent in the final expression:
\begin{align*}
  \frac14 \cdot \frac{\log((1+\delta_0)pn)}{\log(1/\rho(A))}
  & \leq \frac14 \cdot \frac{\log\del*{\frac{1+\delta_0}{1-\delta_1} \lambda_1(A)}}{\log(1/\rho(A))} \\
  & = \frac14 \cdot \del*{ \frac{\log\del*{\frac{1+\delta_0}{1-\delta_1}}}{\log(1/\rho(A))} + \frac{\log(\lambda_1(A))}{\log(\lambda_1(A)) - \log(\lambda(A))} } \\
  & \leq \frac14 \cdot \del*{ \frac{\frac{\delta_0+\delta_1}{1-\delta_1}}{\log(1/\delta_1)} + \frac1{1 - \frac{\log(\lambda(A))}{\log(\lambda_1(A))}} } \\
  & \leq \frac14 \cdot \del*{ \frac{\frac{\delta_0+\delta_1}{1-\delta_1}}{\log(1/\delta_1)} + \frac1{1 - \frac{\log((2+\varepsilon+\delta_2)\sqrt{pn})}{\log((1-\delta_1)pn)}} } \\
  & = \frac14 \cdot \del*{ \frac{\frac{\delta_0+\delta_1}{1-\delta_1}}{\log(1/\delta_1)} + \frac1{1 - \frac{\log(2+\varepsilon+\delta_2) + \frac12\log(pn)}{\log(1-\delta_1) + \log(pn)}} } \\
  & = \frac14 \cdot \del*{ o(1) + \frac1{1 - \frac{1/2 + o(1)}{1 - o(1)} } } \\
  & = \frac14 \cdot \del*{ 2 + o(1) } = \frac12 + o(1) .
\end{align*}
Therefore the number of nodes visited is
\begin{equation*}
  O\del*{ n/\delta }^{\frac12 + o(1) } .
  \qedhere
\end{equation*}
\end{proof}

\subsection{Proof of \Cref{claim:admissible}}
\begin{proof}
\label{proof:admissiblepair}
    If $\mathrm{comp}_G(s)=\mathrm{comp}_G(t)$, the algorithm returns a simple $s$-$t$ path $\Pi\subseteq E(G)$ upon the first run of generation phase. Note that $E(G)\subseteq E(G^*)$ implies $\Pi\subseteq E(G^*)$.
 
    Otherwise, given that $(G,G^*)$ is an admissible pair, by definition, there exists a connected component $C$ of $G$ such that, for every vertex $u \in V$ with $N_{G^*}(u)\neq\emptyset$,
    \begin{align*}
        \pi_u^{G^*}\big( N_{G^*}(u) \cap V(C) \big) \ge \gamma.
    \end{align*}
    The algorithm repeatedly queries $\mathcal{O}_{G^*}(s)$ and $\mathcal{O}_{G^*}(s)$ until it finds a neighbor of both $s$ and $t$ denoted by $v_s$ and $v_t$ in $C$. Since $C$ is a connected component of $G$, in the generation phase, a BFS in $G$ yields a simple path $\Pi_C\subseteq E(G)$ from $v_s$ to $v_t$. Note that $E(G)\subseteq E(G^*)$ implies $\Pi_C\subseteq E(G^*)$.  Moreover, both of the $(s,v_s)$ and $(v_t,t)$ edges are returned by the retrieval oracle on $G^*$, and therefore, lie in $E(G^*)$. Thus, a path will be found during the generation phase. 
    Note that if the oracle returns $\bot$ either $s$ or $t$, then there can be no $s$-$t$ path in $G^*$, so the algorithm outputs \textsc{NO}.  
   
    It remains to bound the number $Q$ of retrieval calls. For any endpoint $x$ with $N_{G^*}(x)\neq\emptyset$, $\gamma$-admissibility implies that each query to $\mathcal{O}_{G^*}(x)$ hits $C$ with probability at least $\gamma$, independently of past failures. Therefore the expected number of query calls for each end point is bounded by $\gamma$, that is, $\mathbb{E}[Q_x]\le 1/\gamma$. For $Q:=Q_s+Q_t$ by linearity of expectation $\mathbb{E}[Q]\le 2/\gamma$ when a path exists; thus, the pair $(G,G^*)$ is $2/\gamma$-retrieval friendly.   
\end{proof}

\subsection{Proof of \Cref{theorem:admissibleErods}}
\begin{proof}
    Since \(G \sim \ER(n,p\eta)\) and \(np\eta > 1\), the standard Erd\H{o}s--R\'enyi giant-component theorem \cite{Frieze_Karoński_2015} implies that with high probability there is a unique giant \(C\) with
    \(|C|=(\gamma \pm  o(1))n\) with $\gamma=1-e^{-n p \eta \gamma}$, and all other components are \(O(\log n)\). 
    Similarly, with high probability $G^*$ is connected.
    The generation process is equivalent to first sampling $G\sim\ER(n,p\eta)$ and then, to form $G^*$, adding each edge not present in $G$ independently with probability $q=\frac{p-p\eta}{1-p\eta}$. This ensures $G^*$ is a valid $\ER(n,p)$ graph. This allows us to first condition on the realization of $G$ (and thus its giant component $C$) and then analyze the properties of $G^*$.
    Note that in Erd\H{o}s--R\'enyi model for all $u$,  $\pi_u$ is uniform over $N_{G^*}(u)$, and admissibility reduces to 
    \begin{align*}
        \frac{\lvert N_{G^*}(u) \cap V(C) \rvert}{\lvert N_{G^*}(u) \rvert} \ge \gamma.
    \end{align*}
    
    Fix constant $\alpha \in (0, 1/5]$. 
    For any vertex $u \in V$, its degree $|N_{G^*}(u)|$ follows a binomial distribution $\mathrm{Bin}(n-1,p)$.
    Since \(p(n-1)\ge\log n\), Chernoff bounds give, for each fixed \(u\)
    \begin{align*}
        \Pr\big( |N_{G^*}(u)| > (1+\alpha)(n-1)p \big) &\le \exp\left(-\frac{\alpha^2}{3}(n-1)p\right) \\
    \end{align*}
    Similarly, conditioned on $G$, for each vertex not in the giant component the number of its neighbors within the giant component, $|N_{G^*}(u) \cap V(C)|$, follows $\mathrm{Bin}(|C|,p)$
    \begin{align*}
        \Pr\big( |N_{G^*}(u) \cap V(C)| < (1-\alpha)|C|p \mid G \big) &\le \exp\left(-\frac{\alpha^2}{3}|C|p\right)
    \end{align*}
    To ensure these bounds hold for all vertices simultaneously, we apply a union bound. The probability of the first event failing for at least one vertex is at most
    
    \begin{align*}
        \Pr\!\Big(\exists u:\ |N_{G^*}(u)|>(1+\alpha)(n-1)p\Big) &\le n \cdot \exp\left(-\frac{\alpha^2 C_0}{3}\frac{(n-1)\log n}{n}\right) \\
        \Pr\!\Big(\exists u:\ |N_{G^*}(u) \cap V(C)|<(1-\alpha)|C|p \ \Big|\, G\Big) &\le n \cdot \exp\left(-\frac{\alpha^2 C_0}{3}\frac{ |C| \log n}{n}\right) 
    \end{align*}
    
    For sufficiently large $n$, $|C| \ge \frac{n \gamma}{2} $. For $C_0 > \frac{18}{\gamma \alpha^2}$ probabilities above are $o(n^{-2})$, and the lower bound ratio for any vertex $u$:
    $$
    \frac{|N_{G^*}(u) \cap V(C)|}{|N_{G^*}(u)|} \ge \frac{(1-\alpha)|C|p}{(1+\alpha)(n-1)p} = \frac{1-\alpha}{1+\alpha} \cdot \frac{|C|}{n-1} \ge \gamma/3.
    $$
\label{proof:admissibleErdos}
\end{proof}

\subsection{Proof of \Cref{thm:disjoinpaths}}
\label{proof:coupons}
\begin{proof}
    Fix an endpoint $x\in\{s,t\}$. Consider $K$ bins, one per partition $i\in[K]$, such that bin $i$ corresponds to $C_i$, and throw one ball per querying the retriever $\mathcal{O}_{G^*}(x)$.
    A ball occupies bin $i$ if the returned neighbor lies in $V(C_i)$.
    By \Cref{def:layerwise-global}, for every $i$,
    \[
    \Pr[\text{ball occupies bin }i]= \pi_u^{G^*}\big( N_{G^*}(x)\cap V(C_k)\big) \ge \gamma
    \]
    Note that one ball may occupy multiple bins, that is, the same vertex can lie in many $C_i$ which only helps cover faster.
    After $t$ throws, for any fixed $i$ the probability bin $i$ is still empty is at most $(1-\gamma)^t\le e^{-\gamma t}$.
    By a union bound over the $K$ bins,
    \[
    \Pr[\exists\text{ empty bin after }t\text{ balls}]\ \le\ K e^{-\gamma t}.
    \]
    Doing this for both endpoints yields an expected total of at most $O(\frac{\log K}{\gamma})$ queries. Then, for both endpoint $\{s,t\}$ bin $i$ is occupied by anchor $v_s^{(i)},v_t^{(i)}\in V(C_i)$. 
    Since $C_i$ is a connected component of $G_i=(V,E_i)$, a BFS yields a simple path $\Pi_{C_i}\subseteq E(G)$ from $v_s$ to $v_t$. Note that $E(G)\subseteq E(G^*)$ implies $\Pi_{C_i} \subseteq E(G^*)$. Moreover, every $(s,v_s^{(i)})$ and $(v_t^{(i)},t)$ edges are returned by the retrieval oracle on $G^*$, and therefore, lie in $E(G^*)$. Thus, every $(s,v_s^{(i)}) \circ \Pi_{C_i} \circ (v_t^{(i)},t)$  is also in $G^*$.  Moreover, since $\Pi_{C_i}$ uses only edges of $E_i$, and $\{E_i\}_{i=1}^K$ partitions $E$, the internal segments $\{\Pi_{C_i}\}_{i=1}^K$ are pairwise edge disjoint in $G$ and $G^*$.
\end{proof}

\section{Future Directions}
\label{FutureDirections}
\subsection{Empirical Directions}

Another extension is to provide empirical evidence that matches our theory. An easy route is to implement synthetic experiments showing that the model accuracy has an inflection point: low with to few queries and high with enough queries. However, this does not shed light on the impact of parametric knowledge in real LLMs. Also, we already have much evidence that RAG and tool use work well with modern LLMs. 

To validate the necessity of dense parametric knowledge, it would be ideal to train models on multiple mixtures of pre-training corpora, crafted to have different proportions of a target domain. For example, one could train on a mix of general purpose web data and selectively chosen data in a niche domain, like medical or law documents. 

\subsection{Theoretical Directions}
In this work, we focused mainly on finding a path between two vertices $s, t$ and our examples on an  Erd\H{o}s–R\'enyi random graph with $\eta$ retention threshold on the prior.  It is natural to seek query-complexity thresholds for other graph-theoretic tasks under a partially observed prior. The relevant phenomenon is a prior sensitive phase transition, that is, a critical retention level $\eta(P)$ at which a task $P$ switches from requiring $\omega(1)$ queries to allowing $O(1)$ expected queries. In general, there are many sub-linear graph and matrix questions that we can study with prior knowledge. For example, see \cite{beame2020edge, feige2004sums, feige2021tight, racz2019finding, rashtchian2020vector, rashtchian2021average, chen2020nearly} and references therein. Importantly, our work opens up new questions, where we can study how the query complexity changes based on the knowledge $G$ instead of starting with no information about $G^*$.
This includes problems with more global dependencies, such as Minimum Spanning Tree recovery. Note that a natural extension to finding a (shortest) path between two vertices is to consider a set of $M$ input vertices $(s_1,\ldots,s_M)$ and ask whether the learner can efficiently recover a (minimum) spanning tree connecting them all.

What is the tightest possible lower bound on the query complexity for finding a tree spanning input vertices $(s_1,\ldots,s_M)$? This bound should be characterized as a function of the structural properties and densities of both the pretrained graph $G$ and the ground truth graph $G^*$?
What structural properties beyond admissibility of $(G,G^*)$ guarantee a constant upper bound on the expected number of retrieval queries for recovering a tree spanning $(s_1,\ldots,s_M)$? 

Moreover, problems concerning local structure, like triangle detection and counting are interesting to explore. 
Another interesting future direction is the observation model that generates $G$ and $G^*$. In this work we use i.i.d. edge retention, but other realistic mechanisms include radius-dependent thinning in random geometric $k$-NN graphs, which models conserving local edges while suppressing long edges, and adversarial deletions. Each induces a different critical $\eta(P)$ and poses open problems at the interface of random graph theory and query complexity.

\end{document}